\documentclass{article}

\usepackage{microtype}
\usepackage{graphicx}
\usepackage{subcaption}
\usepackage{booktabs} % for professional tables

\usepackage{hyperref}

\usepackage[accepted]{icml2026}

\usepackage{amsmath}
\usepackage{amssymb}
\usepackage{mathtools}
\usepackage{amsthm}

\usepackage[capitalize,noabbrev]{cleveref}

\theoremstyle{plain}
\newtheorem{theorem}{Theorem}[section]

\newtheorem{proposition}{Proposition}[section]

\theoremstyle{definition}
\newtheorem{assumption}{Assumption}[section]

\theoremstyle{remark}

\usepackage[disable,textsize=tiny]{todonotes}
\usepackage{mdframed}
\usepackage{pifont}
\usepackage{enumitem}
\definecolor{teal}{rgb}{0.0, 0.5, 0.5}

\usepackage{array}
\newcolumntype{I}{!{\vrule width 1pt}}
\usepackage[table]{xcolor}  % 支持 gray!20 / HTML 等
\usepackage{colortbl}       % 提供 
\usepackage{bm}
\usepackage{multirow}
\usepackage{arydshln}
\newcommand{\lora}{LoRA}

\newcommand{\ourmethod}[2]{\textbf{Ours}}

\makeatletter
\newcommand{\thickhline}{%
    \noalign {\ifnum 0=`}\fi \hrule height 1pt
    \futurelet \reserved@a \@xhline
}
\hypersetup{
    pagebackref,
    breaklinks,
    colorlinks,
    citecolor=teal,   % 引用文献变成蓝绿色 [1]
    urlcolor=blue,    % 网址变成蓝色 
    linkcolor=black   % 公式、图表跳转变成黑
}
\usepackage{tcolorbox}  % 加载tcolorbox 宏包
\tcbuselibrary{theorems}
\icmltitlerunning{Shift-Dependent Asymmetry: Orthogonal Inverse LoRA for Federated Medical Segmentation}

\begin{document}

\twocolumn[
  \icmltitle{Shift-Dependent Asymmetry: Orthogonal Inverse Low-Rank Adaptation for Federated Medical Segmentation}



  \icmlsetsymbol{equal}{*}

  \begin{icmlauthorlist}
    \icmlauthor{Xingyue Zhao}{equal,pumch,ia}
    \icmlauthor{Wenke Huang}{equal,ntu}
    \icmlauthor{Linghao Zhuang}{equal,cams,ia}
    \icmlauthor{Haoran Wu}{ia}
    \icmlauthor{Anwen Jiang}{xju}
    \icmlauthor{Zhifeng Wang}{ant}
    \icmlauthor{Wenwen He}{whu}
    \icmlauthor{Ming Feng}{pumch}
    \icmlauthor{Mang Ye}{whu}
    \icmlauthor{Bo Xu}{ia}
  \end{icmlauthorlist}

  \icmlaffiliation{pumch}{Department of Neurosurgery, Peking Union Medical College Hospital, Chinese Academy of Medical Sciences and Peking Union Medical College}
  \icmlaffiliation{ia}{The Key Laboratory of Cognition and Decision Intelligence for Complex Systems, Institute of Automation, Chinese Academy of Sciences}
  \icmlaffiliation{ntu}{College of Computing and Data Science, Nanyang Technological University, Singapore}
  \icmlaffiliation{cams}{Institute of Basic Medical Sciences, Chinese Academy of Medical Sciences and Peking Union Medical College, Beijing, China}
  \icmlaffiliation{xju}{Xinjiang University}
  \icmlaffiliation{ant}{Ant Group, China}
  \icmlaffiliation{whu}{Wuhan University}

  \icmlcorrespondingauthor{Haoran Wu}{wuhaoran2018@ia.ac.cn} % TODO: fill in real email
  \icmlcorrespondingauthor{Ming Feng}{fengming@pumch.cn} % TODO: fill in real email
  \icmlcorrespondingauthor{Bo Xu}{xubo@ia.ac.cn} % TODO: fill in real email

  \icmlkeywords{Machine Learning, ICML}

  \vskip 0.3in
]



\printAffiliationsAndNotice{\icmlEqualContribution}  % first two authors contributed equally

\begin{abstract}
Low-Rank Adaptation (LoRA) enables efficient federated fine-tuning of segmentation foundation models for medical imaging. However, most federated LoRA methods adopt a uniform aggregation rule, which breaks under the encoder–decoder asymmetry in medical segmentation: the encoder is dominated by appearance shifts, while the decoder is dominated by supervision variations. This mismatch entangles shared anatomy with site-specific biases and harms generalization. To address this, we propose Inverse Asymmetric Tuning (IAT). IAT aligns adaptation with heterogeneity sources by personalizing module-specific components in the encoder to absorb appearance shifts and in the decoder to accommodate site-dependent supervision, while retaining a shared pathway for transferable consensus. However, structural separation alone is insufficient under LoRA’s bilinear parameterization, where multiplicative coupling can still cause site-specific updates to leak into the shared direction. We therefore introduce a Subspace Orthogonality Regularizer that penalizes shared–local collinearity in the effective update space, mitigating leakage without extra communication. Experiments show consistent improvements over strong federated LoRA and parameter-efficient FL baselines.
\end{abstract}

\begin{figure}[t]
	\centering
	\includegraphics[width=\linewidth]{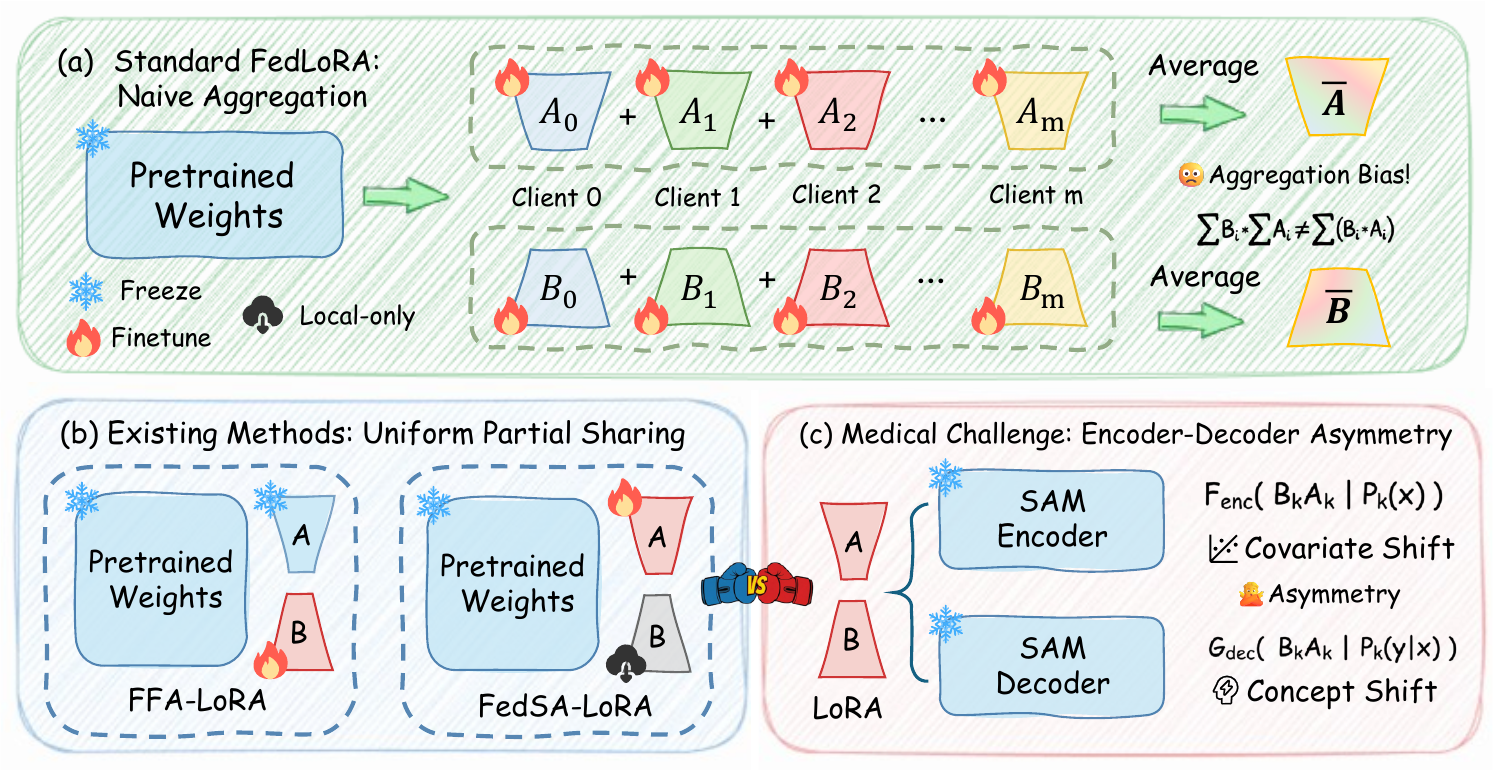} 
	% \caption{\small Overview of the proposed personalized federated SAM framework. Left: Process flow of our framework. Right: Detailed architecture of SAM and the L-MoE structure.}
    \caption{\small \textbf{Illustration of the motivation.} 
(a) \textbf{Standard FedLoRA} employs naive aggregation of all parameters, suffering from high communication costs and severe client drift. 
(b) \textbf{Existing Methods} (e.g., uniform partial sharing) enforce a static splitting rule across the entire network, overlooking the distinct roles of different modules. 
(c) \textbf{Medical Challenge: Encoder-Decoder Asymmetry.} We identify a critical structural misalignment: the encoder is primarily dominated by \textit{covariate shift} (appearance variations), while the decoder is dominated by \textit{concept shift} (supervision targets).}
	\label{motivation}
    \vspace{-1.5em}
\end{figure}

\section{Introduction}
Medical image segmentation is fundamental to clinical decision-making, ranging from organ delineation to tumor assessment~\cite{cao2023large,hu2025ai,zhu2026medeyes}. Developing robust segmentation models typically requires diverse data from multiple centers~\cite{hu2025ai_1,zhu2025pathology}; however, aggregating patient images is often restricted by strict privacy regulations and data governance policies~\cite{jiang2022harmofl}. Federated Learning (FL)~\cite{FedAvg_AISTATS17,huang2026federated} addresses this challenge by enabling collaborative training across institutions without exchanging raw data. However, the inherent data heterogeneity across sites often limits the performance of conventional lightweight models. To address this, there is a growing trend to deploy Segmentation Foundation Models (e.g., SAM)~\cite{kirillov2023segment}, as their robust pre-trained representations offer superior generalization against domain shifts~\cite{liu2024fedfms,zhang2026dsfedmed}.

Yet, fine-tuning such massive backbones in federated settings imposes prohibitive communication and computational burdens on local computing resources. To make collaborative training feasible, Parameter-Efficient Fine-Tuning (PEFT) has become indispensable. In particular, Low-Rank Adaptation (LoRA)~\cite{hu2022lora} has emerged as the dominant strategy for federated learning~\cite{yi2023pfedlora, zhang2024towards,liaosplitting}, as it significantly reduces transmission overhead by updating only lightweight low-rank matrices while freezing the pre-trained backbone. Accordingly, we center our investigation on the Federated LoRA paradigm.

However, standard aggregation of LoRA weights suffers from an inherent inconsistency due to its bilinear parameterization ($\Delta W = B A$). Since matrix multiplication is non-linear, averaging decomposed factors on the server generally fails to reconstruct the average of the effective updates. This introduces a coupling term that biases the global model:
% \begin{equation}\small
% \setlength\abovedisplayskip{0pt}\setlength\belowdisplayskip{0pt}
% \begin{aligned}
% \bar B\,\bar A
% &= \frac{1}{K}\sum_{k=1}^K B_kA_k \;+\; \mathcal{E},\\
% \mathcal{E}
% &:= \frac{1}{K}\sum_{k=1}^K (B_k-\bar B)(A_k-\bar A),
% \end{aligned}
% \end{equation}
% \begin{equation}\small
% \setlength\abovedisplayskip{0pt}\setlength\belowdisplayskip{0pt}
% \bar B\,\bar A = \frac{1}{K}\sum_{k=1}^K \bigl[ B_kA_k + (B_k-\bar B)(A_k-\bar A) \bigr]
% \end{equation}
\begin{equation}\small
\setlength\abovedisplayskip{0pt}\setlength\belowdisplayskip{0pt}
\bar B\,\bar A = \frac{1}{K}\sum_{k=1}^K \Bigl[ B_kA_k + \tcbhighmath[colback=green!8, colframe=white, boxrule=0pt, arc=0pt, left=0mm, right=0mm, top=0mm, bottom=0mm]{(B_k-\bar B)(A_k-\bar A)} \Bigr]
\end{equation}
% The last term represents the interference caused by conflicting local updates. In non-IID scenarios, this deviation becomes significant, leading to biased global weights. 
Here, the term $(B_k-\bar B)(A_k-\bar A)$ represents the interference caused by conflicting local updates. In non-IID scenarios, this deviation becomes significant, leading to biased global weights. To mitigate this, recent Federated LoRA approaches aim to isolate these conflicts by retaining specific parameters locally. Current methods implement this by freezing one factor (e.g., ~\cite{zhang2023lora}) or by selectively sharing specific matrices (e.g., ~\cite{guo2024selective, wang2025pfedsam}). Yet, these methods typically enforce a uniform splitting rule across the entire network. This rigid strategy fails to address the asymmetric heterogeneity in medical segmentation. Mathematically, the LoRA update $\Delta W_k = B_k A_k$ is driven by structurally distinct statistical dependencies across modules:
% \begin{equation}\small
% \setlength\abovedisplayskip{0pt}\setlength\belowdisplayskip{0pt}
% \begin{aligned}
% \Delta W_k \sim \begin{cases} \mathcal{F}_{enc}\big(B_k A_k \mid P_k(\mathbf{x})\big) & \text{Encoder (Covariate Shift)} \\ \mathcal{G}_{dec}\big(B_k A_k \mid P_k(\mathbf{y}|\mathbf{x})\big) & \text{Decoder (Concept Shift)} \end{cases}
% \end{aligned}
% \end{equation}
\begin{equation}\small
\setlength\abovedisplayskip{0pt}\setlength\belowdisplayskip{0pt}
\begin{aligned}
\Delta W_k \sim \begin{cases} 
% 第一行：Encoder (使用淡灰/默认色)
\mathcal{F}_{enc}\big(\tcbhighmath[colback=yellow!10, colframe=white, boxrule=0pt, arc=0pt, left=0mm, right=0mm, top=0mm, bottom=0mm]{B_k A_k \mid P_k(\mathbf{x})}\big) & \text{Encoder (Covariate Shift)} \\ 
% 第二行：Decoder (使用淡绿色 green!8)
\mathcal{G}_{dec}\big(\tcbhighmath[colback=green!8, colframe=white, boxrule=0pt, arc=0pt, left=0mm, right=0mm, top=0mm, bottom=0mm]{B_k A_k \mid P_k(\mathbf{y}|\mathbf{x})}\big) & \text{Decoder (Concept Shift)} 
\end{cases}
\end{aligned}
\end{equation}
The encoder LoRA is primarily sensitive to acquisition shifts (changes in input distribution $P(\mathbf{x})$), whereas the decoder LoRA confronts supervision variations (changes in conditional distribution $P(\mathbf{y}|\mathbf{x})$). Applying a uniform split ignores these distinct statistical roles, leaving site-specific biases entangled with shared anatomical knowledge. To this end, we aim to develop a structure-aware tuning framework that moves beyond homogeneous constraints. This raises the following challenge:
% \colorbox{mydarkblue!50}{\parbox{0.98\textwidth}{
% \textbf{\textit{Challenge}} \hypertarget{Q1}{\textbf{\uppercase\expandafter{\romannumeral1})}} \textit{\textbf{How to effectively inject multi-task knowledge into the low-rank subspace?}}
% }}
\definecolor{cadetblue}{RGB}{95,158,160} % Define CadetBlue color
\definecolor{keywordcolor}{RGB}{178,34,34} % Define FireBrick color for keywords
\definecolor{mydarkblue}{RGB}{211, 218, 234}

\begin{mdframed}[backgroundcolor=cadetblue!10, linewidth=0.8pt, linecolor=cadetblue!80, roundcorner=5pt]
\textit{\textcolor{keywordcolor}{How to structurally disentangle the adaptation strategy to reconcile opposing heterogeneity shifts?}}
\end{mdframed}
% Instead of enforcing a uniform splitting rule across the entire network~\cite{zhang2023lora,guo2024selective,wang2025pfedsam}, we model the encoder as being dominated by covariate (appearance) shift and the decoder by concept (supervision) shift. Our theoretical analysis (Proposition 1, Section 3) reveals a shift-dependent preference: minimizing the reconstruction error under covariate shift favors localizing the input-side factor ($A$), whereas concept shift favors localizing the output-side factor ($B$). We empirically validate this inversion by measuring the inter-client similarity of LoRA matrices across the encoder and decoder modules. The results, illustrated in Figure 2, confirm a distinct "crossover" pattern: in upstream encoder blocks, $A$ exhibits markedly lower similarity than $B$, while in downstream decoder blocks, the trend reverses. Motivated by this crossover, we propose Inverse Asymmetric Tuning (IAT). IAT establishes a structure-aware calibration mechanism. 
Instead of enforcing a uniform splitting rule across the entire network~\cite{zhang2023lora,guo2024selective,wang2025pfedsam}, we distinguish the roles of the encoder and decoder, modeling the former as being dominated by covariate (appearance) shift and the latter by concept (supervision) shift. Our theoretical analysis (Proposition 1, Section 3) formally derives a shift-dependent preference: minimizing the reconstruction error under covariate shift necessitates localizing the input-side factor ($A$), whereas concept shift requires localizing the output-side factor ($B$). Guided by this theoretical derivation, we propose Inverse Asymmetric Tuning (IAT) to explicitly align the tuning strategy with these distinct structural requirements. Empirical results, illustrated in Figure \ref{m}, corroborate our theoretical analysis, confirming that this structure-aware calibration significantly outperforms uniform baselines. Guided by this analysis, IAT implements an inverse allocation strategy. Specifically, it selectively personalizes the input-projection factors ($A$) in the encoder to absorb covariate (acquisition) shifts, while personalizing the output-mapping factors ($B$) in the decoder to adapt to concept (supervision) shifts. This design explicitly aligns the parameter optimization focus with the dominant heterogeneity source of each module.

% Instead of a static rule, it enforces an inverse allocation paradigm that selectively personalizes the input-projection factors ($A$) in the encoder to filter acquisition shifts, while personalizing the output-mapping factors ($B$) in the decoder to align with supervision variations. This explicitly couples the parameter optimization landscape with the dominant heterogeneity source of each module. 

However, structural separation alone does not guarantee optimization independence. While IAT correctly allocates the parameters, the bilinear dependency of LoRA ($\Delta W = BA$) induces an intrinsic coupling during training. As gradients backpropagate through this product, the update direction of the shared matrix is inevitably modulated by the local matrix, creating a channel for implicit leakage. Consequently, the distinct subspaces defined by IAT may gradually collapse, allowing site-specific biases to contaminate the shared global model. This raises the second challenge:
% \vspace{-em}
\begin{mdframed}[backgroundcolor=cadetblue!10, linewidth=0.8pt, linecolor=cadetblue!80, roundcorner=5pt]
\textit{\textcolor{keywordcolor}{How to ensure the functional independence of the decoupled subspaces to prevent optimization interference?}}
\end{mdframed}
% \vspace{-0.8em}
To address this, we analyze the gradient dynamics of the bilinear interaction (Proposition 2, Section 3). Our analysis demonstrates that without constraints, the shared and local updates naturally tend toward collinearity. To counteract this, we propose the Subspace Orthogonality Regularizer (SOR). SOR imposes a soft geometric constraint that forces the shared and local subspaces to remain orthogonal. This mechanism rectifies the gradient flow, ensuring that the shared model aggregates only universal representations while site-specific variations remain strictly isolated. Our main contributions are summarized as follows:
\begin{itemize}[leftmargin=*, nosep]
\item[{\ding{182}}] \textbf{Re-examining LoRA Allocation in FL Medical Segmentation.} Our findings indicate that the prevalent uniform splitting strategy exhibits a severe structural mismatch in medical segmentation. We reveal that the optimal parameter sharing preference is not static, but structurally inverts between the encoder and decoder due to their opposing heterogeneity patterns.
\item[{\ding{183}}] \textbf{Novel Structure-Aware Framework for Dual Decoupling.} Building on the phenomenon of allocation inversion, we propose a unified framework to ensure comprehensive disentanglement. We effectively address the structural misalignment and optimization coupling inherent in federated medical LoRA training.
\item[{\ding{184}}] \textbf{Theoretical Guarantees and Experimental Validation.} We provide theoretical guarantees for the shift-dependent allocation preference and the necessity of orthogonality constraints. Extensive experiments demonstrate the effectiveness and robustness of our framework compared to state-of-the-art methods.
\end{itemize}

\paragraph{Conflict of Interest Disclosure.}
The authors declare no financial conflicts of interest related to this work.

\section{Related Works}
\noindent\textbf{Data Heterogeneity in Federated Learning.}
Federated learning (FL) enables collaborative training without sharing raw data, yet real-world deployments typically suffer from non-i.i.d.\ distributions~\cite{huang2022learn,hsieh2020non}. This statistical heterogeneity manifests primarily in three forms: label distribution skew~\cite{kairouz2021advances}, feature distribution skew~\cite{FPL_CVPR23}, and quantity skew~\cite{wang2020tackling,FedProx_MLSys2020}. Among these, we specifically focus on \textit{feature skew} (appearance shift). In this setting, clients share a consistent label space but differ substantially in input distributions due to variations in acquisition devices or environmental conditions, leading to severe representation drift.
Strategies to mitigate such heterogeneity broadly fall into two paradigms: \textit{optimization enhancement} and \textit{personalized adaptation}. The former stabilizes global training by rectifying local optimization dynamics, for instance, by restricting client drift via proximal regularization~\cite{FedProx_MLSys2020} or correcting update bias using control variates~\cite{MOON_CVPR21}. However, since enforcing a single global model often proves insufficient under severe feature skew, the latter paradigm pursues personalized federated learning (PFL)~\cite{yang2024fedas}. These approaches abandon the one-model-fits-all constraint by structurally isolating domain-specific components, such as decoupling prediction heads~\cite{collins2021exploiting} or maintaining local normalization statistics~\cite{li2021fedbn}, thereby allowing effective adaptation to diverse local distributions.
Such feature heterogeneity is particularly pronounced in medical image segmentation, where diverse acquisition protocols create distinct visual styles that degrade model performance~\cite{jiang2022harmofl,zhu2024advancing,shao2025rethinking,zhao2026divide}. Recently, foundation models like Segment Anything (SAM)~\cite{kirillov2023segment,zhao2024sam} have emerged as powerful backbones due to their robust general-purpose representations~\cite{liu2024fedfms,asokan2024federated}. However, integrating these large models into federated pipelines presents a dual challenge: achieving parameter-efficient adaptation under strict communication constraints while robustly handling significant cross-site appearance shifts~\cite{wang2025pfedsam,hu2025federated}.

\noindent\textbf{LoRA in Federated Learning.} Foundation models have recently become a strong backbone choice in federated learning~\cite{fan2023fate}, but full-parameter fine-tuning is often impractical due to the high communication and on-device training cost~\cite{xu2024fwdllm}. Parameter-efficient fine-tuning (PEFT) addresses this by freezing the backbone and updating only a small parameter subset~\cite{houlsby2019parameter,liu2022few,hu2024learn}. Notably, Low-Rank Adaptation (LoRA)~\cite{hu2022lora} is widely favored in FL for its communication efficiency, achieved by transmitting only lightweight low-rank factors~\cite{yi2023pfedlora,liu2025differentially,long2024dual}. However, the bilinear nature of LoRA ($\Delta W = BA$) complicates aggregation: simply averaging local factors at the server fails to reconstruct the average effective update, a discrepancy that worsens under non-i.i.d.\ data~\cite{sun2024improving}. To mitigate this, distinct aggregation methods have been proposed. One stream employs \textit{asymmetric sharing}~\cite{guo2024selective,zhang2023lora} or \textit{residual reparameterization}~\cite{yanfederated}, selectively aggregating factors or restructuring updates to minimize cross-client interference. Another stream focuses on \textit{adaptive configurations}, such as importance-aware splitting rules~\cite{liaosplitting,zhao2025fedlora} or heterogeneous rank allocations~\cite{cho2024heterogeneous,chen2024rbla,wang2024flora}, to handle diverse client constraints. It is worth noting that most current Federated LoRA strategies are developed for decoder-only Large Language Models (LLMs). Conversely, medical segmentation employs encoder--decoder architectures (e.g., SAM) with dense pixel-level supervision, creating distinct structural constraints. Consequently, directly applying LLM-centric protocols may be suboptimal, particularly as heterogeneity tends to affect the encoder and decoder modules in different ways.
\begin{figure*}[t]
% 	\vspace{-4pt}
	\begin{center}
    \includegraphics[width=0.95\linewidth]{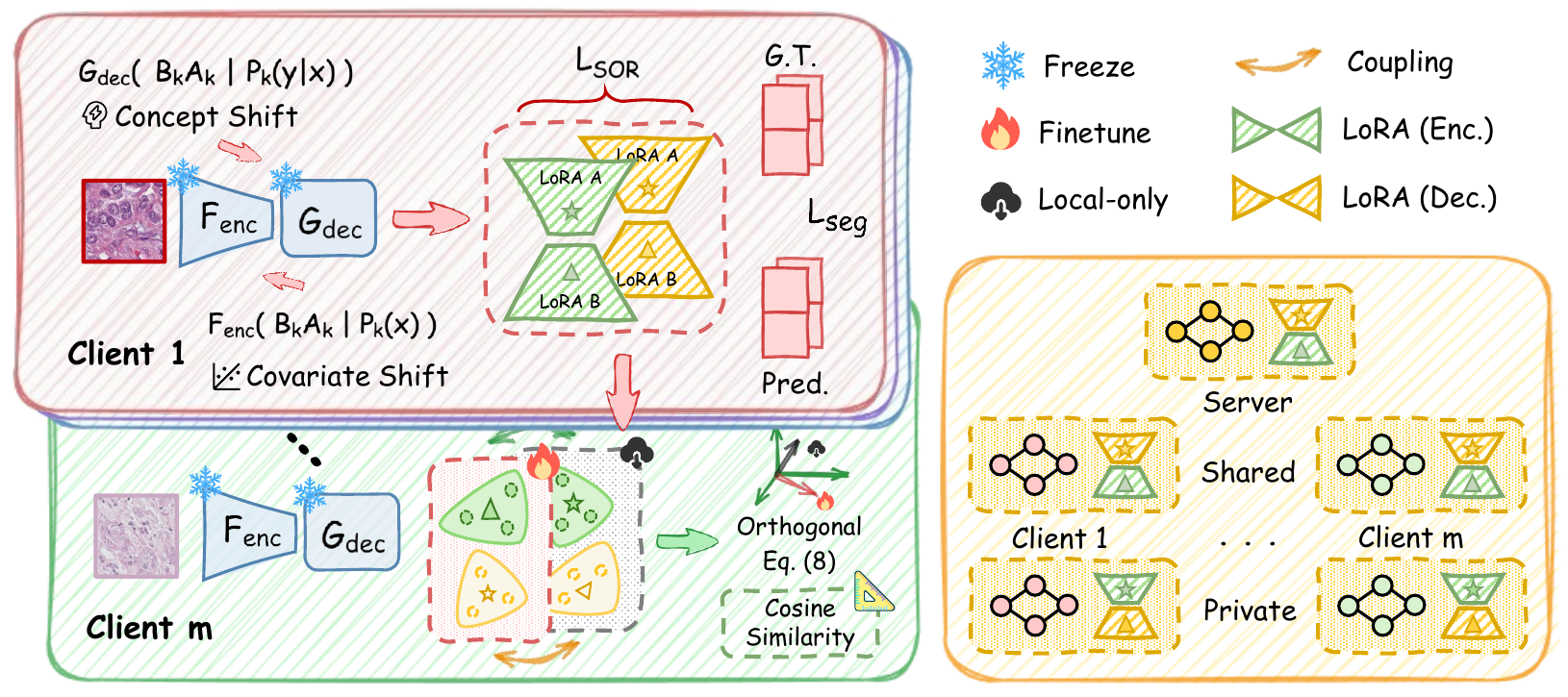}
		% \put(-109,93){\scriptsize\cref{eq:celoss}}
		% \put(-124,55){\scriptsize\cref{eq:cpclloss}}
		% \put(-68,25){\scriptsize\cref{eq:upcrloss}}
	\end{center}
	% \vspace{-5pt}
	\captionsetup{font=small}
    \caption{\small\textbf{Architecture illustration} of the proposed framework. To address the structural misalignment between covariate and concept shifts, we implement \textbf{Inverse Asymmetric Tuning (IAT)} which allocates trainable LoRA factors inversely across the encoder and decoder. To prevent the coupling of shared knowledge with local biases, we apply a \textbf{Subspace Orthogonality Regularizer (SOR)} during local updates. The server aggregates only the shared components to derive a generalized global model while preserving local personalization. Zoom in for details.}
	\label{fig:framework}
	% \vspace{-15pt}
\end{figure*}

\section{Methodology}
\subsection{Preliminary} 
\noindent\textbf{Federated Medical Segmentation.} We consider a federated segmentation task with $K$ clients. Each client $k$ holds a private dataset $\mathcal{D}_k = \{(x_i, y_i)\}_{i=1}^{N_k}$ drawn from a distinct distribution $\mathcal{P}_k$. The goal is to minimize the global objective $\min_{\Theta} \sum_{k=1}^K p_k \mathcal{L}_k(\Theta)$.
We formulate the segmentation network $\mathcal{F}$ as a composition of an Encoder ($\mathcal{E}$) and a Decoder ($\mathcal{D}$), i.e., $\mathcal{F} = \mathcal{D} \circ \mathcal{E}$.
This architecture faces dual heterogeneity in medical imaging:
(i) Acquisition Shift ($P_k(x)$), caused by varying scanner protocols, primarily affects the Encoder which extracts low-level features;
(ii) Supervision Shift ($P_k(y|x)$), caused by distinct annotation standards, primarily affects the Decoder which maps features to semantic predictions.

\noindent\textbf{Low-Rank Adaptation (LoRA).} To efficiently adapt the pre-trained parameters $W_0 \in \mathbb{R}^{d_{out} \times d_{in}}$ (in both $\mathcal{E}$ and $\mathcal{D}$), LoRA introduces a low-rank update $\Delta W = B A$, where $B \in \mathbb{R}^{d_{out} \times r}$ and $A \in \mathbb{R}^{r \times d_{in}}$ ($r \ll d_{in}, d_{out}$). The forward pass is:
\begin{equation}
\setlength\abovedisplayskip{0pt}\setlength\belowdisplayskip{0pt}
h = W_0 x + B A x,
\label{eq:lora}
\end{equation}
where $x$ denotes the input to the layer. Here, $A$ acts as the input projector compressing information into the rank-$r$ subspace, while $B$ acts as the output projector reconstructing it back to the functional dimension.
\subsection{Proposed Method}
To address the critical issues of \textbf{structural mismatch} in parameter allocation and \textbf{implicit optimization coupling} under dual heterogeneity, we propose a dual-decoupling framework for federated medical segmentation. As illustrated in Figure~\ref{fig:framework}, FedIAT effectively disentangles general and site-specific knowledge by integrating \textbf{Inverse Asymmetric Tuning (IAT)} to structurally align LoRA factors with module-specific shifts, and a \textbf{Subspace Orthogonality Regularizer (SOR)} to functionally suppress gradient leakage during training. This synergy ensures that global knowledge and site-specific priors are decoupled in both structure and optimization dynamics.

\subsubsection{Inverse Asymmetric Tuning}
\noindent\textbf{Motivation.} \textit{The optimal parameter allocation strategy in federated segmentation is not uniform but structurally inverts depending on the dominant heterogeneity type.}
Existing PFL-LoRA methods typically enforce a uniform splitting rule, ignoring the structural asymmetry in medical segmentation. The encoder is primarily exposed to acquisition shifts (input domain), whereas the decoder faces supervision shifts (output domain). To rigorously explain why this necessitates a flexible strategy, we analyze the approximation error of a linear surrogate layer.
\begin{proposition}[Shift-Dependent Preference]
\label{proposition1}
Consider a linear surrogate layer $y = (W_0 + BA)x$ with rank $r$.
(i) Under Covariate Shift (input subspace rotation $x_k = R_k x_{gen}$), minimizing the approximation error favors sharing $B$ and personalizing $A_k$ to align the client-specific input row space.
(ii) Under Concept Shift (target mapping rotation $y_k = T_k y_{gen}$), minimizing the error favors sharing $A$ and personalizing $B_k$ to align the client-specific output column space.
\end{proposition}
\vspace{-1.5em}
\begin{proof}
See Appendix \ref{sec:proof_of_prop1}.
\end{proof}
\vspace{-1.0em}
% \noindent\textbf{Method.} Guided by Proposition xxx and the empirical crossover pattern, we propose Inverse Asymmetric Tuning (IAT), which discards uniform splitting in favor of a module-aware protocol. Specifically, let $\mathcal{E}$ and $\mathcal{D}$ denote the sets of LoRA-adapted layers in the encoder and decoder, respectively. For encoder layers ($l \in \mathcal{E}$), dominated by acquisition shift, we adopt a "Local-$A$ / Shared-$B$" strategy. Clients optimize the input projector $A_k$ locally, while the output projector $B$ is aggregated. The server update rule is defined as:
% \begin{equation}
% B_{agg}^{(t+1, l)} \leftarrow \sum_{k=1}^K p_k B_k^{(t+1, l)}, \quad \forall l \in \mathcal{E}, \quad \text{with } A_k^{(t+1,l)} \text{ retained locally.}
% \label{eq:iat_encoder}
% \end{equation}
% Conversely, for decoder layers ($l \in \mathcal{D}$), dominated by supervision shift, we invert this strategy to "Shared-$A$ / Local-$B$". Here, the input projector $A$ is aggregated to maintain a consistent shared feature subspace, while the output projector $B_k$ remains local. The corresponding protocol follows:
% \begin{equation}
% A_{agg}^{(t+1, l)} \leftarrow \sum_{k=1}^K p_k A_k^{(t+1, l)}, \quad \forall l \in \mathcal{D}, \quad \text{with } B_k^{(t+1,l)} \text{ retained locally.}
% \label{eq:iat_decoder}
% \end{equation}
% This structural inversion explicitly aligns parameter roles with the source of heterogeneity, effectively absorbing acquisition-driven appearance variations in the encoder while accommodating supervision-induced prediction variations in the decoder.
\noindent\textbf{Method.} Guided by Proposition 1 and the empirical ``crossover'' pattern, we propose \textbf{Inverse Asymmetric Tuning (IAT)}, which discards uniform splitting in favor of a module-aware protocol. Specifically, let $\mathcal{E}$ and $\mathcal{D}$ denote the sets of LoRA-adapted layers in the encoder and decoder, respectively.

For encoder layers ($l \in \mathcal{E}$), dominated by acquisition shift, we adopt a Local-$A$ / Shared-$B$ strategy. Clients optimize the input projector $A_k$ locally to filter site-specific imaging artifacts, while the output projector $B$ is aggregated. The server update rule is:
\begin{equation}
\small
\setlength\abovedisplayskip{4pt}
\setlength\belowdisplayskip{4pt}
B_{agg}^{(t+1, l)} \leftarrow \sum_{k=1}^K p_k B_k^{(t+1, l)}, \quad \forall l \in \mathcal{E},
\label{eq:iat_encoder}
\end{equation}
while the input projector $A_k^{(t+1,l)}$ is retained locally.

Conversely, for decoder layers ($l \in \mathcal{D}$), dominated by supervision shift, we invert this strategy to Shared-$A$ / Local-$B$. Here, the input projector $A$ is aggregated to maintain a consistent shared feature subspace:
\begin{equation}
\small
\setlength\abovedisplayskip{4pt}
\setlength\belowdisplayskip{4pt}
A_{agg}^{(t+1, l)} \leftarrow \sum_{k=1}^K p_k A_k^{(t+1, l)}, \quad \forall l \in \mathcal{D},
\label{eq:iat_decoder}
\end{equation}
while the output projector $B_k^{(t+1,l)}$ is retained locally to adapt to divergent annotation standards. This structural inversion explicitly aligns parameter roles with the source of heterogeneity.

\subsubsection{Subspace Orthogonality Regularizer}
\noindent\textbf{Motivation.}
\textit{Structural decoupling alone is insufficient; the bilinear parameterization induces implicit gradient coupling, necessitating functional constraints to prevent leakage.} 
Although IAT structurally separates parameters, the optimization dynamics remain entangled due to the multiplicative nature of LoRA ($\Delta W = B A$). The gradient update for the shared factor is functionally dependent on the local factor ($G_k A_k^\top$), creating a channel for heterogeneity leakage. We formalize this interaction through standard SGD decomposition:
\begin{proposition}[Bilinear Leakage in Federated Updates]
Let $\Delta W_k = B A_k$ where $B$ is shared and $A_k$ is local. The aggregated global update for $B$ with learning rate $\eta$ can be decomposed as:
\begin{equation}
\small
\setlength\abovedisplayskip{4pt}
\setlength\belowdisplayskip{4pt}
B^{(t+1)} = B^{(t)} - \eta \underbrace{\sum_{k=1}^K p_k G_k \bar{A}^\top}_{\text{Common Drift}} - \eta \underbrace{\sum_{k=1}^K p_k G_k (A_k - \bar{A})^\top}_{\text{Heterogeneity Leakage}},
\label{eq:bilinear_leakage}
\end{equation}
where $\bar{A} = \sum p_k A_k$, and $G_k := \nabla_{\Delta W}\mathcal{L}_k$ denotes the gradient w.r.t. the adaptation matrix $\Delta W$. The last term represents the leakage where local deviations $(A_k - \bar{A})$ contaminate the shared update. An analogous decomposition holds symmetrically for the decoder (shared $A$, local $B$) by swapping the roles of $(A, B)$.
\end{proposition}

\vspace{-1.5em}
\begin{proof}
See Appendix \ref{sec:proof_of_prop2}.
\end{proof}
\vspace{-1.0em}
\noindent\textbf{Method.} To suppress this leakage, we introduce SOR, which discourages alignment between the shared update direction and the local drift direction in a compact $r \times r$ proxy space. Here, let $A_{0,k}^{(t,l)}, B_{0,k}^{(t,l)}$ denote the detached parameter anchors at the start of round $t$, and let $\delta A_{k}^{(t,l)}, \delta B_{k}^{(t,l)}$ be the exponential moving averages (EMA) of the private-factor drifts within the round. We define the rank-efficient proxies with explicit stop-gradient (sg) operations as:
\begin{equation}
\small
\setlength\abovedisplayskip{4pt}
\setlength\belowdisplayskip{4pt}
\begin{aligned}
P_{sh,k}^{(t,l)} &= \big(B_{k}^{(t,l)} - B_{0,k}^{(t,l)}\big)^{\top} B_{0,k}^{(t,l)}, \\
P_{lo,k}^{(t,l)} &= \operatorname{sg}\!\left[ A_{0,k}^{(t,l)} \big(\delta A_{k}^{(t,l)}\big)^{\top} \right], \quad \forall l\in\mathcal{E}; \\
Q_{sh,k}^{(t,l)} &= \big(A_{k}^{(t,l)} - A_{0,k}^{(t,l)}\big) A_{0,k}^{(t,l)\top}, \\
Q_{lo,k}^{(t,l)} &= \operatorname{sg}\!\left[ B_{0,k}^{(t,l)\top} \delta B_{k}^{(t,l)} \right], \quad \forall l\in\mathcal{D}.
\end{aligned}
\label{eq:sor_proxies}
\end{equation}
We then minimize the squared normalized Frobenius inner product between these proxies. By minimizing Eq. (\ref{eq:sor_loss}), SOR produces gradients primarily for the shared factors, ensuring they evolve orthogonally to the local drifts while leaving personalization unconstrained:
\begin{equation}
\small
\setlength\abovedisplayskip{4pt}
\setlength\belowdisplayskip{4pt}
\begin{split}
\mathcal{L}_{\mathrm{SOR}}^{(k)} &= \sum_{l \in \mathcal{E}} \left( \frac{\langle P_{sh,k}^{(t,l)}, P_{lo,k}^{(t,l)} \rangle_F}{\|P_{sh,k}^{(t,l)}\|_F \, \|P_{lo,k}^{(t,l)}\|_F + \epsilon} \right)^2 \\
&\quad + \sum_{l \in \mathcal{D}} \left( \frac{\langle Q_{sh,k}^{(t,l)}, Q_{lo,k}^{(t,l)} \rangle_F}{\|Q_{sh,k}^{(t,l)}\|_F \, \|Q_{lo,k}^{(t,l)}\|_F + \epsilon} \right)^2.
\end{split}
\label{eq:sor_loss}
\end{equation}
Each client $k$ optimizes the total objective $\mathcal{L}_{\mathrm{total}}^{(k)} = \mathcal{L}_{\mathrm{seg}} + \lambda \mathcal{L}_{\mathrm{SOR}}^{(k)}$, and the server aggregates factors according to the IAT protocol.

\subsection{Convergence Analysis}

We establish the convergence guarantees of the proposed framework within the standard non-convex federated optimization analysis. To explicitly model the partial sharing mechanism, we parameterize the full optimization space as $\Theta := (\Theta^{\mathrm{sh}}, \{\Theta_k^{\mathrm{lo}}\}_{k=1}^K)$, where $\Theta^{\mathrm{sh}}$ represents the aggregated LoRA factors (shared across clients) and $\Theta_k^{\mathrm{lo}}$ denotes the client-specific local factors. The global objective is formulated as:
\begin{equation}\small
\setlength\abovedisplayskip{0pt}\setlength\belowdisplayskip{0pt}
	\min_{\Theta} F(\Theta) := \sum_{k=1}^K p_k \mathcal{L}_k(\Theta_k)
\end{equation}
here the local objective is defined as $\mathcal{L}_k(\Theta_k) := \mathcal{L}^{\text{seg}}_k(\Theta^{\mathrm{sh}}, \Theta_k^{\mathrm{lo}}) + \lambda \mathcal{L}^{\text{SOR}}_k(\Theta_k)$. Here, $p_k$ is the relative weight of client $k$. The training process spans $R$ communication rounds. In each round, participating clients perform $E$ steps of local SGD with stepsize $\eta$ to update both shared and local components, after which the server aggregates only $\Theta^{\mathrm{sh}}$.

We adopt the following standard assumptions to characterize the optimization landscape.

\begin{assumption}[$L$-Smoothness]
\label{assumption1}
For each client $k$, the local objective $\mathcal{L}_k$ is differentiable and $L$-smooth with respect to $\Theta_k$. That is, for any parameters $\Theta_1, \Theta_2$:
\begin{equation}\small
\setlength\abovedisplayskip{0pt}\setlength\belowdisplayskip{0pt}
    \mathcal{L}_k(\Theta_1) \le \mathcal{L}_k(\Theta_2) + \langle \nabla \mathcal{L}_k(\Theta_2), \Theta_1 - \Theta_2 \rangle + \frac{L}{2}\|\Theta_1 - \Theta_2\|_F^2.
\end{equation}
\end{assumption}

\begin{assumption}[Bounded Gradients]
\label{assumption2}
Let $g_{k,t} = \nabla \mathcal{L}_k(\Theta_{k,t}; \xi_{k,t})$ be the unbiased stochastic gradient sampled from client $k$ at step $t$. The expected squared norm is uniformly bounded, i.e., $\mathbb{E}\|g_{k,t}\|_F^2 \le G^2$ for all $k, t$.
\end{assumption}

\begin{assumption}[Non-Degenerate LoRA Factors]
\label{assumption3}
To ensure sufficient gradient flow through the low-rank bottleneck, we assume the LoRA factors remain non-degenerate (i.e., $\sigma_{\min} \ge \delta > 0$). Consequently, there exist constants $c_A, c_B > 0$ such that the projection of the true gradient onto the local update direction $U_{k,t}$ satisfies:
\begin{equation}\small
\setlength\abovedisplayskip{0pt}\setlength\belowdisplayskip{0pt}
    \langle \nabla \mathcal{L}_k(\Theta_{k,t}), U_{k,t} \rangle \ge (c_A + c_B) \|\nabla \mathcal{L}_k(\Theta_{k,t})\|_F^2,
\end{equation}
where the inner product is defined over the corresponding LoRA-parameter subspace.
\end{assumption}

Assumptions \ref{assumption1} and \ref{assumption2} are widely adopted in federated optimization \citep{li2019convergence}. Assumption \ref{assumption3} is specific to low-rank parameterization, guaranteeing that the update direction remains a valid descent direction by preventing the collapse of the optimization landscape.

Based on these assumptions, we derive the following convergence rate (detailed proof in Appendix \ref{proof_of_theo}).

\begin{theorem}
\label{theorem1}
Let Assumptions \ref{assumption1}--\ref{assumption3} hold. Let $T$ be the total number of local iterations. By setting the learning rate $\eta = \min\{\bar{\eta}, \sqrt{\frac{4D}{MT}}\}$, where $D$ bounds the initial suboptimality and $M = 2LC_2G^2$ is a problem-dependent constant (defined in Appendix), we have:
\begin{equation}\small
\setlength\abovedisplayskip{0pt}\setlength\belowdisplayskip{0pt}
    \frac{1}{K T} \sum_{k=1}^K \sum_{t=0}^{T-1} \mathbb{E}\Bigl[\|\nabla \mathcal{L}_k(\Theta_{k,t})\|_F^2\Bigr]
    \;\le\;
    \frac{2}{c_A + c_B} \sqrt{\frac{D M}{T}}.
\end{equation}
\end{theorem}

Theorem \ref{theorem1} indicates that our method achieves an $\mathcal{O}(1/\sqrt{T})$ convergence rate to a stationary point. This matches the standard rate of FedAvg in non-convex settings up to lower-order aggregation drift terms (see Appendix), confirming that our asymmetric partial sharing strategy preserves theoretical convergence guarantees.

\begin{figure}[t]
	\centering
	\includegraphics[width=\linewidth]{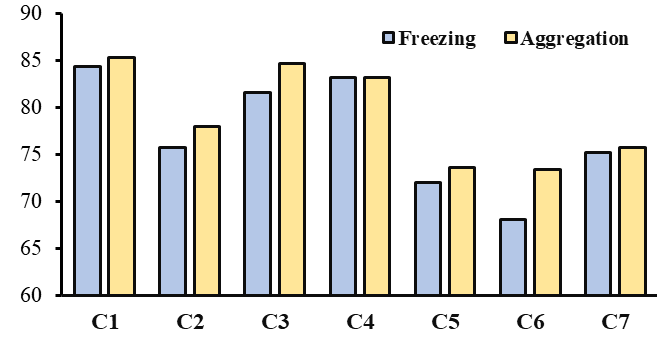} 
\caption{\small \textbf{Preliminary ablation on LoRA configuration.} We compare the performance of applying LoRA to the \textit{Encoder-only} versus the \textit{Encoder+Decoder}. The results demonstrate that, unlike classification tasks, medical segmentation requires adapting the decoder to reconstruct precise pixel-level details, validating our choice to inject LoRA into both modules.}
	\label{lora_val}
    \vspace{-1.5em}
\end{figure}

\begin{table*}[t]\small
\centering
\captionsetup{font=small}
\caption{{
\textbf{Comparison with State-of-the-Art Fine-Tuning Solutions} on Histology nuclei and Fundus photography images.
The optimal and sub-optimal results are denoted by boldface and underlining. {\color{red}$\uparrow$} means improved accuracy compared with the sub-optimal results. 
}}
\scriptsize{
\resizebox{\linewidth}{!}{
\setlength\tabcolsep{3.pt}
\renewcommand\arraystretch{1.1}
\begin{tabular}{r||ccccccc|c||cccc|c}
\hline\hline
\rowcolor{gray!20}
 &
\multicolumn{8}{c|}{\textbf{Histology nuclei}} &
\multicolumn{5}{c}{\textbf{Fundus photography images}}
\\
\cline{2-14}
\rowcolor{gray!20}
\multirow{-2}{*}{Methods\quad}
& \textsc{Adrena} & \textsc{Esophagus} & \textsc{Bile-duct} & \textsc{Uterus} & \textsc{MoNuSAC} & \textsc{TNBC} & \textsc{MoNuSeg} & \textit{Avg}
& \textsc{REFUGE} & \textsc{ORIGA} & \textsc{G1020} & \textsc{Drishti-GS1} & \textit{Avg}
\\
\hline\hline

\multicolumn{14}{l}
{\textcolor{gray!80}{\textit{LoRA Rank=8}}}
\\
\rowcolor{gray!10} FedIT
& 85.25 & \textbf{78.98} & \underline{84.69} & \underline{83.28} & 72.97 & 72.16 & 76.42 & 79.11
& 88.85 & 84.60 & \underline{78.15} & 57.74 & 77.33
\\

FLoRA
& 79.45 & 71.67 & 77.09 & 79.77 & 60.49 & 63.84 & 65.24 & 71.08
& 79.89 & 80.98 & 69.50 & 71.71 & 75.52
\\

\rowcolor{gray!10} FedSA
& 84.89 & 77.07 & 83.97 & 82.29 & \underline{75.68} & 75.65 & \underline{81.10} & \underline{80.09}
& 88.69 & 83.67 & \textbf{79.15} & \underline{80.64} & \underline{83.04}
\\

FFA-LoRA
& 82.21 & 69.99 & 77.76 & 80.30 & 65.40 & 3.59 & 0.74 & 54.29
& \underline{89.27} & 81.12 & 72.60 & 29.07 & 68.02
\\

\rowcolor{gray!10} FedDPA
& 83.55 & 76.10 & 81.65 & 81.35 & 74.36 & \underline{76.93} & \textbf{81.51} & 79.35
& 88.08 & 84.70 & 72.75 & 78.28 & 80.95
\\

LoRA-FAIR
& 84.13 & 75.64 & 82.73 & 83.10 & 72.50 & 65.70 & 75.00 & 76.97
& 89.24 & 84.78 & 77.86 & 70.78 & 80.67
\\

\rowcolor{gray!10} FlexLoRA
& \underline{85.71} & \underline{77.30} & 82.88 & \textbf{84.13} & 74.35 & 70.73 & 72.36 & 78.21
& \textbf{89.38} & \underline{85.40} & 77.53 & 65.72 & 79.51
\\

FRLoRA
& 83.87 & 76.74 & 81.39 & 83.16 & 72.74 & 70.48 & 71.35 & 77.11
& 88.90 & 84.69 & 77.21 & 60.55 & 77.83
\\
\hline
\rowcolor[HTML]{D7F6FF}
\textbf{Ours}
& \textbf{85.80} & 77.10 & \textbf{85.55} & 82.79 & \textbf{78.01} & \textbf{79.49} & 81.06 & \textbf{81.40}
& 89.00 & \textbf{85.47} & \underline{78.15} & \textbf{85.43} & \textbf{84.52}
\\
\hline
\hline
\multicolumn{14}{l}{\textcolor{gray!80}{\textit{LoRA Rank=16}}}
\\
\rowcolor{gray!10} FedIT
& 84.48 & 76.81 & \underline{85.03} & 83.25 & 74.93 & 77.79 & 77.49 & 79.97
& \textbf{89.67} & 85.80 & \underline{80.00} & 41.52 & 74.25
\\

FLoRA
& 83.06 & 72.76 & 79.03 & 81.30 & 63.68 & 59.62 & 68.29 & 72.53
& 83.44 & 83.37 & 71.99 & 74.84 & 78.41
\\

\rowcolor{gray!10} FedSA
& 84.13 & \underline{78.20} & 83.39 & 81.66 & 75.62 & \underline{78.42} & \textbf{82.86} & \underline{80.61}
& 89.39 & 86.04 & 79.36 & \underline{83.91} & \underline{84.68}
\\

FFA-LoRA
& 81.96 & 69.41 & 77.23 & 80.00 & 63.80 & 3.03 & 0.87 & 53.76
& 89.34 & 80.95 & 72.46 & 26.82 & 67.39
\\

\rowcolor{gray!10} FedDPA
& 84.17 & 77.64 & 83.53 & 82.47 & \underline{75.63} & 76.95 & \underline{81.75} & 80.31
& 88.76 & 84.61 & 73.65 & 77.42 & 81.11
\\

LoRA-FAIR
& 84.23 & 78.13 & 84.32 & 83.41 & 75.38 & 74.25 & 77.70 & 79.63
& 88.73 & \textbf{86.51} & 75.46 & 64.78 & 78.87
\\

\rowcolor{gray!10} FlexLoRA
& \textbf{85.77} & \textbf{79.52} & 81.90 & 83.32 & 74.83 & 71.11 & 74.01 & 78.64
& 89.48 & 86.19 & 78.35 & 71.87 & 81.47
\\

FRLoRA
& 84.43 & 78.01 & 80.94 & \textbf{83.70} & 74.90 & 71.11 & 74.20 & 78.18
& 89.42 & 84.75 & 78.78 & 59.77 & 78.18
\\

\rowcolor[HTML]{D7F6FF}
\textbf{Ours}
& \underline{85.69} & 78.19 & \textbf{85.52} & \underline{83.51} & \textbf{75.96} & \textbf{80.49} & 81.12 & \textbf{81.49}
& \underline{89.61} & \underline{86.48} & \textbf{80.30} & \textbf{86.43} & \textbf{85.70}
\\

\hline
\end{tabular}}}

\label{tab:compare_sota_image_capation}
\end{table*}

\begin{table}[t]\small
\captionsetup{font=small}
\caption{
\textbf{Ablation study of key components} on the retinal fundus segmentation task. 
}
\centering
{
\resizebox{\columnwidth}{!}{
\setlength\tabcolsep{5pt}
\renewcommand\arraystretch{1.1}
\begin{tabular}{cc||ccccIc}
\hline \thickhline
\rowcolor{gray!20} &
& \multicolumn{4}{cI}{\textbf{Fundus photography images}}
&  \\
\cline{3-7} 
\rowcolor{gray!20}
\multirow{-2}{*}{IAT} & \multirow{-2}{*}{SOR}  
& \textsc{REFUGE} & \textsc{ORIGA} & \textsc{G1020} & \textsc{Drishti-GS1} & \textit{Avg}
\\
\hline\hline

\multicolumn{2}{c||}{\lora{}} 
& 88.85 & 84.60 & 78.15 & 57.74 & 77.33 
\\ 
\hdashline

\ding{51} & 
& \textbf{89.13} & 85.39 & 77.71 & 81.64 & 83.47
\\

\rowcolor[HTML]{D7F6FF}
\ding{51} & \ding{51}
& 89.00 & \textbf{85.47} & \textbf{78.15} & \textbf{85.43} & \textbf{84.52}
\\
\hline
\end{tabular}}}

\label{tab:ablation_module}
\end{table}

% \begin{table}[t]\small
% \centering
% {
% \resizebox{\columnwidth}{!}{
% \setlength\tabcolsep{6pt}
% \renewcommand\arraystretch{1.1}
% \begin{tabular}{l||cIc}
% \hline \thickhline
% \rowcolor{gray!20}
% \textbf{Method}
% & \textbf{Trainable Params.}
% & \textbf{Per-round Communicated Params.}
% \\
% \hline\hline

% FedIT      & -- & -- \\
% FLoRA      & -- & -- \\
% FedSA      & -- & -- \\
% FFA-LoRA   & -- & -- \\
% FedDPA     & -- & -- \\
% LoRA-FAIR  & -- & -- \\
% FlexLoRA   & -- & -- \\
% FRLoRA     & -- & -- \\
% \hline
% \rowcolor[HTML]{D7F6FF}
% \textbf{Ours} & \textbf{--} & \textbf{--} \\
% \hline
% \end{tabular}}}
% \captionsetup{font=small}
% \caption{
% \textbf{System Efficiency Comparison} of fine-tuning methods, reporting trainable parameters and per-round communicated parameters.
% }
% \label{tab:system_efficiency}
% \end{table}
\begin{table}[t]\small
\captionsetup{font=small}
\caption{
\textbf{System Efficiency Comparison} reporting Trainable Parameters and Per-round Communicated Parameters.
}
\centering
{
\resizebox{\columnwidth}{!}{
\setlength\tabcolsep{6pt}
\renewcommand\arraystretch{1.1}
\begin{tabular}{l||cIc}
\hline \thickhline
\rowcolor{gray!20}
\textbf{Method} & \textbf{Trainable Param.} & \textbf{Per-round Communicated Param.}
\\
\hline\hline
\rowcolor{gray!10} FedIT     & 0.39M & 0.78M \\
FLoRA     & 4.35M & 29.94M \\
\rowcolor{gray!10} FedSA     & 0.39M & 0.25M \\
FFA-LoRA  & 0.26M & 0.53M \\
\rowcolor{gray!10} FedDPA    & 4.65M & 8.71M \\
LoRA-FAIR & 0.39M & 0.78M \\
\rowcolor{gray!10} FlexLoRA  & 0.39M & 0.78M \\
FRLoRA    & 0.29M & 21.82M \\
\hline
\rowcolor[HTML]{D7F6FF}
\textbf{Ours} & 0.39M & 0.55M \\
\hline
\end{tabular}}}

\label{tab:system_efficiency}
\end{table}

\section{Experiments} 
\subsection{Experiment Settings}
\noindent \textbf{Datasets}. To comprehensively evaluate the effectiveness of our proposed framework, we conduct extensive experiments on challenging medical image segmentation benchmarks under federated learning settings. We utilize two widely used and publicly available tasks, as detailed below.
\begin{itemize}[leftmargin=*]
        \setlength{\itemsep}{0pt}
	\setlength{\parsep}{-2pt}
	\setlength{\parskip}{-0pt}
	\setlength{\leftmargin}{-10pt}
	\vspace{-8pt}
% \item \textbf{Histology nuclei segmentation}: We assembled four public datasets (TCIA~\cite{hou2020dataset}, CRC~\cite{awan2017glandular}, KIRC~\cite{irshad2014crowdsourcing}, TNBC~\cite{naylor2018segmentation}) for binary segmentation of cell nuclei in histology images.
\item \textbf{Histology nuclei segmentation:} We evaluate binary nuclei segmentation in a federated, cross-domain setting comprising seven clients. The setup includes four clients derived from specific tissue types within the PanNuke dataset (Adrenal gland, Esophagus, Bile duct, and Uterus)~\cite{gamper2019pannuke,gamper2020pannuke}. The remaining three clients correspond to the independent MoNuSeg~\cite{kumar2017dataset}, MoNuSAC2020~\cite{verma2021monusac2020}, and TNBC~\cite{naylor2018segmentation} datasets. This benchmark is characterized by significant heterogeneity across clients in terms of tissue origin, staining procedures, and scanner variations between different centers. To ensure a unified binary task, all instance masks were converted to foreground/background masks.
\item \textbf{Fundus photography images segmentation}: This task involves the joint segmentation of the optic disc (OD) and optic cup (OC) across four decentralized clients representing different public fundus datasets: REFUGE~\cite{orlando2020refuge}, ORIGA-light~\cite{zhang2010origa}, Drishti-GS1~\cite{sivaswamy2015comprehensive}, and G1020~\cite{bajwa2020g1020}. Client heterogeneity primarily arises from diverse image acquisition conditions across different clinical sites, variations in patient cohorts, and distinct imaging protocols (e.g., the specific 30$^\circ$ field-of-view in Drishti-GS1 versus standard 45$^\circ$ protocols).
\vspace{-8pt}
\end{itemize}
We conduct experiments in a federated setting by treating each dataset as an individual client, mimicking real-world scenarios where each medical institution acts as a distinct participant. While these clients share a unified label space, their image appearances vary significantly due to diverse image acquisition protocols and equipment specifications. Following \cite{liu2024fedfms,wang2025pfedsam}, we adopt the same preprocessing pipelines and data splitting protocols. 

\noindent \textbf{Implementation Details.} We present our implementation details from three aspects:
\begin{itemize}[leftmargin=*]
        \setlength{\itemsep}{0pt}
	\setlength{\parsep}{-2pt}
	\setlength{\parskip}{-0pt}
	\setlength{\leftmargin}{-10pt}
	\vspace{-8pt}
\item \textbf{Model:} We adopt the Segment Anything Model (SAM) \cite{kirillov2023segment} as our foundation. Specifically, we utilize the \textbf{SAM ViT-B} variant initialized with weights pre-trained on the SA-1B dataset. The input images are resized to $1024 \times 1024$ and normalized following standard SAM protocols. Our framework is implemented in PyTorch and executed on NVIDIA A100 GPUs. For the federated setting, the training procedure spans 200 global communication rounds. In each round, the server employs a \textbf{weighted aggregation strategy} (based on client sample sizes) to update the global model. Locally, each client performs training for 1 epoch using the Adam optimizer with a batch size of 4. The optimizer's momentum parameters are set to $\beta_1=0.9$ and $\beta_2=0.999$. All clients share identical hyperparameter configurations to ensure fair comparisons.

\item \textbf{LoRA Configuration:} We inject LoRA rank decomposition matrices into the query and value projection layers of the Transformer blocks. Crucially, we apply LoRA to both the image encoder and the mask decoder. Unlike other tasks where the encoder dominates, medical image segmentation is a dense prediction task that relies heavily on the decoder to reconstruct pixel-level details from high-level features. To validate this, we conducted a preliminary ablation (see Fig.~\ref{lora_val}), comparing freezing the backbone while fine-tuning LoRA on the encoder-only versus the encoder-decoder. The results demonstrate that adapting the decoder is essential for capturing boundary details, yielding superior performance.

\item \textbf{Evaluation Metric:} Following prior works \cite{jiang2022harmofl}, we employ the \textbf{Dice Similarity Coefficient (DSC)} to quantitatively evaluate the segmentation accuracy. The DSC measures the overlap between the predicted segmentation mask and the ground truth.
\vspace{-8pt}
\end{itemize}

\noindent \textbf{Compared Baselines:} We compare our proposed framework with several state-of-the-art Federated Parameter-Efficient Fine-Tuning (FedPEFT) methods: (1) \textbf{FedIT}~\cite{zhang2024towards}. (2) \textbf{FLoRA}~\cite{wang2024flora}. (3) \textbf{FedSA}~\cite{guo2024selective}. (4) \textbf{FFA-LoRA}~\cite{sun2024improving}. (5) \textbf{FedDPA}~\cite{long2024dual}. (6) \textbf{LoRA-FAIR}~\cite{bian2025lora}. (7) \textbf{FlexLoRA}~\cite{bai2024federated}. (8) \textbf{FRLoRA}~\cite{yanfederated}. Note that since these methods were originally designed for Large Language Models (LLMs), we adapted their core mechanisms to our segmentation backbone to ensure a fair comparison.

\subsection{Comparison to State-of-the-Arts}
Table~\ref{tab:compare_sota_image_capation} presents the quantitative performance comparison on the Histology nuclei and Fundus photography datasets. Our results demonstrate that the proposed framework achieves superior performance across diverse heterogeneous clients. Existing FedPEFT methods often struggle to effectively align divergences caused by severe domain shifts, leading to performance instability. For instance, FFA-LoRA suffers from catastrophic degradation on specific clients (e.g., 3.59\% on TNBC), while FedIT fails to generalize to distinct domains like Drishti-GS1 (57.74\%). In contrast, our method successfully maintains robust generalization capabilities under these challenging conditions. By effectively decoupling covariate and concept shifts, our framework achieves the highest average DSC of \textbf{81.40\%} and \textbf{84.52\%} on the two tasks, surpassing the second-best methods by \textbf{1.31\%} and \textbf{1.48\%}, respectively. Specifically, on the most distinct domain (Drishti-GS1), our method outperforms the baseline FedIT by a remarkable margin of over 27\%, validating the effectiveness of our structure-aware tuning strategy in handling medical heterogeneity.

\subsection{Diagnostic Analysis}
\noindent \textbf{Hyper-parameter Study}. Fig.~\ref{hyper} illustrates the performance variations across both \textbf{Histology and Fundus datasets}. We observe a consistent trend on both modalities: introducing the regularization term ($\lambda > 0$) improves performance compared to negligible weighting ($10^{-6}$), confirming the necessity of the proposed SOR module. Specifically, the performance peaks at $\lambda = 10^{-4}$, achieving optimal DSC scores of 81.40\% (Histology) and 84.52\% (Fundus). While an overly large $\lambda$ (e.g., $10^{-3}$) leads to a slight decline due to over-regularization, the model maintains robust performance within the optimal range. Therefore, we universally set $\lambda = 10^{-4}$ for all tasks to ensure stable convergence.

\noindent \textbf{Ablation of Key Component}. We evaluate the contribution of each component using four diverse fundus datasets. Table~\ref{tab:ablation_module} presents the results, demonstrating that both components are essential.
The vanilla Federated LoRA struggles with severe domain shifts (e.g., 57.74\% on \textsc{Drishti-GS1}). Incorporating \textbf{IAT} significantly mitigates this, boosting the average DSC to 83.47\%. Furthermore, adding \textbf{SOR} promotes the orthogonality between shared and local features, yielding the optimal average performance of 84.52\%.
\begin{figure}[t]
	\centering
	\includegraphics[width=0.75\linewidth]{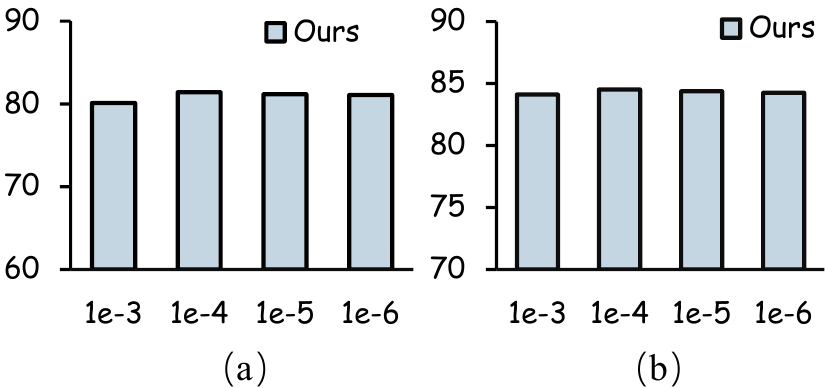} 
\caption{\small \textbf{Sensitivity analysis of $\lambda$.} Left: Histology nuclei. Right: Fundus photography.}
	\label{hyper}
    \vspace{-1.5em}
\end{figure}
% \begin{figure}[t]
% 	\centering
% 	\includegraphics[width=\linewidth]{m1-4.pdf} 
% \caption{\small \textbf{Sensitivity analysis of $\lambda$.} Left: Histology nuclei. Right: Fundus photography.}
% 	\label{m}
%     \vspace{-1.5em}
% \end{figure}

\noindent \textbf{Empirical Validation of Theoretical Analysis.} 
We conducted a pilot study on the Fundus dataset to validate the shift-dependent preference derived in Proposition ~\ref{proposition1}. We compared our proposed structure-aware strategy with standard \textit{Uniform Strategies} (e.g., FedSA) and other heuristic combinations. As shown in Figure~\ref{m}, our method consistently outperforms other configurations, achieving an optimal DSC of 83.47\%. 
Specifically, compared to the strong baseline FedSA (83.04\%), our approach yields a clear improvement by explicitly assigning local modules to absorb specific shifts. 
Moreover, deviating from this theoretical guidance—such as reversing the allocation rule—leads to suboptimal performance (82.54\% or lower). 
These empirical results strongly support our theorem: maximizing parameter efficiency requires inverting the tuning strategy between the encoder and decoder to align with their dominant heterogeneity types.

% \noindent \textbf{System Efficiency.} Table~\ref{tab:system_efficiency} details the efficiency metrics. \textit{(1) Parameter Efficiency:} With only 0.39M trainable parameters, our method remains as lightweight as FedIT and FedSA, while being significantly more storage-efficient than FLoRA (4.35M) and FedDPA (4.65M). \textit{(2) Communication Cost:} By transmitting only shared LoRA factors via IAT, we reduce per-round communication to 0.55M. This outperforms the standard FedIT (0.78M).  \textit{(3) Trade-off:} Although FedSA and FFA-LoRA achieve lower costs (0.25M/0.53M) through aggressive compression, they compromise segmentation accuracy. In contrast, our method achieves SOTA performance with only a marginal communication increase, demonstrating the best efficiency-performance trade-off.
\noindent \textbf{System Efficiency.} Table~\ref{tab:system_efficiency} details the comprehensive efficiency metrics. \textit{(1) Parameter Efficiency:} Our method demonstrates exceptional storage efficiency. With only \textbf{0.39M} trainable parameters, it maintains a lightweight footprint comparable to FedIT and FedSA. This is in sharp contrast to FLoRA (4.35M) and FedDPA (4.65M), which require updating significantly more parameters, imposing heavy storage burdens on local devices. \textit{(2) Communication Cost:} The proposed IAT protocol inherently optimizes bandwidth usage. By decoupling the model and transmitting only the shared LoRA factors (while keeping personalized factors local), we reduce the per-round communication cost to \textbf{0.55M}. This represents a substantial saving compared to the standard FedIT (0.78M), which aggregates all LoRA parameters. \textit{(3) Performance-Efficiency Trade-off:} Although FedSA and FFA-LoRA achieve lower communication costs (0.25M/0.53M) through aggressive compression strategies, they fail to maintain segmentation accuracy in complex medical scenarios. In contrast, our method achieves SOTA performance with only a marginal communication increase, demonstrating the best trade-off between communication overhead and model generalization.

\begin{figure}[t]
	\centering
	\includegraphics[width=0.45\linewidth]{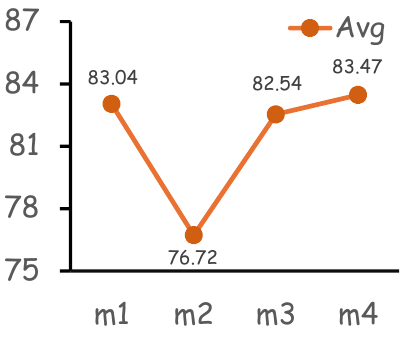}
	\caption{\small \textbf{Empirical validation of theoretical analysis.} 
	We compare different splitting configurations labeled as M1 to M4. 
	\textbf{M1}: Uniform Strategy (e.g., FedSA). 
	\textbf{M2}: Inappropriate Uniform Strategy. 
	\textbf{M3}: Reverse Hybrid Strategy. 
	\textbf{M4}: Proposed Inverse Asymmetric Tuning (IAT). 
	The results show that M4 consistently outperforms other heuristic combinations, validating the shift-dependent structural preference.}
	\label{m}
	\vspace{-1.5em}
\end{figure}

\section{Conclusion}
In this paper, we introduce a novel structure-aware federated tuning framework designed to adapt the Segment Anything Model (SAM) for heterogeneous medical image segmentation. 
Different from uniform splitting strategies, our method employs \textbf{Inverse Asymmetric Tuning} to decouple the optimization landscape: it selectively personalizes the encoder to absorb covariate (appearance) shifts while adapting the decoder to align with concept (supervision) shifts. 
To further mitigate the interference between global and local representations, we incorporate a \textbf{Subspace Orthogonality Regularizer}, which forces the shared and personalized subspaces to remain functionally orthogonal. Extensive experiments on multi-site histology nuclei segmentation and fundus photography images segmentation datasets demonstrate that our approach achieves state-of-the-art performance with superior parameter and communication efficiency.

% \clearpage
% \newpage
\section*{Acknowledgements}
This work is partially supported by the Science and Technology Innovation Program of Xiongan New Area (Grant No.2025XAGG0045), the National High Level Hospital Clinical Research Funding (2025-PUMCH-H-006), and the NTU AI-for-X Postdoctoral Fellowship.

\section*{Impact Statement}
This paper presents work whose goal is to advance the field of Machine
Learning. There are many potential societal consequences of our work, none
which we feel must be specifically highlighted here.


\bibliography{example_paper}

@inproceedings{FedAvg_AISTATS17,
    title={Communication-Efficient Learning of Deep Networks from Decentralized Data},
    author={McMahan, Brendan and Moore, Eider and Ramage, Daniel and Hampson, Seth and y Arcas, Blaise Aguera},
    booktitle={AISTATS},
    pages={1273--1282},
    year={2017}
}

@inproceedings{FedProx_MLSys2020,
  title={Federated optimization in heterogeneous networks},
  author={Li, Tian and Sahu, Anit Kumar and Zaheer, Manzil and Sanjabi, Maziar and Talwalkar, Ameet and Smith, Virginia},
  journal={Proceedings of Machine learning and systems},
  volume={2},
  pages={429--450},
  year={2020}
}

@inproceedings{MOON_CVPR21,
    title={Model-Contrastive Federated Learning},
    author={Li, Qinbin and He, Bingsheng and Song, Dawn},
    booktitle={CVPR},
    pages={10713--10722},
    year={2021}
}

@InProceedings{FPL_CVPR23,
    author    = {Huang, Wenke and Ye, Mang and Shi, Zekun and Li, He and Du, Bo},
    title     = {Rethinking Federated Learning With Domain Shift: A Prototype View},
    booktitle = {CVPR},
    month     = {June},
    year      = {2023},
    pages     = {16312-16322}
}

@inproceedings{huang2022learn,
  title={Learn from others and be yourself in heterogeneous federated learning},
  author={Huang, Wenke and Ye, Mang and Du, Bo},
  booktitle={CVPR},
  pages={10143--10153},
  year={2022}
}

@inproceedings{jiang2022harmofl,
  title={Harmofl: Harmonizing local and global drifts in federated learning on heterogeneous medical images},
  author={Jiang, Meirui and Wang, Zirui and Dou, Qi},
  booktitle={AAAI},
  volume={36},
  number={1},
  pages={1087--1095},
  year={2022}
}

@inproceedings{zhao2024sam,
  title={Sam-Driven Weakly Supervised Nodule Segmentation with Uncertainty-Aware Cross Teaching},
  author={Zhao, Xingyue and Li, Peiqi and Luo, Xiangde and Yang, Meng and Chang, Shi and Li, Zhongyu},
  booktitle={2024 IEEE International Symposium on Biomedical Imaging (ISBI)},
  pages={1--5},
  year={2024},
  organization={IEEE}
}

@inproceedings{kirillov2023segment,
  title={Segment anything},
  author={Kirillov, Alexander and Mintun, Eric and Ravi, Nikhila and Mao, Hanzi and Rolland, Chloe and Gustafson, Laura and Xiao, Tete and Whitehead, Spencer and Berg, Alexander C and Lo, Wan-Yen and others},
  booktitle={Proceedings of the IEEE/CVF international conference on computer vision},
  pages={4015--4026},
  year={2023}
}

@article{hu2022lora,
  title={Lora: Low-rank adaptation of large language models.},
  author={Hu, Edward J and Shen, Yelong and Wallis, Phillip and Allen-Zhu, Zeyuan and Li, Yuanzhi and Wang, Shean and Wang, Lu and Chen, Weizhu and others},
  journal={ICLR},
  volume={1},
  number={2},
  pages={3},
  year={2022}
}

@inproceedings{hsieh2020non,
  title={The non-iid data quagmire of decentralized machine learning},
  author={Hsieh, Kevin and Phanishayee, Amar and Mutlu, Onur and Gibbons, Phillip},
  booktitle={International Conference on Machine Learning},
  pages={4387--4398},
  year={2020},
  organization={PMLR}
}

@article{kairouz2021advances,
  title={Advances and open problems in federated learning},
  author={Kairouz, Peter and McMahan, H Brendan and Avent, Brendan and Bellet, Aur{\'e}lien and Bennis, Mehdi and Bhagoji, Arjun Nitin and Bonawitz, Kallista and Charles, Zachary and Cormode, Graham and Cummings, Rachel and others},
  journal={Foundations and trends{\textregistered} in machine learning},
  volume={14},
  number={1--2},
  pages={1--210},
  year={2021},
  publisher={Now Publishers, Inc.}
}

@article{wang2020tackling,
  title={Tackling the objective inconsistency problem in heterogeneous federated optimization},
  author={Wang, Jianyu and Liu, Qinghua and Liang, Hao and Joshi, Gauri and Poor, H Vincent},
  journal={Advances in neural information processing systems},
  volume={33},
  pages={7611--7623},
  year={2020}
}

@inproceedings{liu2024fedfms,
  title={Fedfms: Exploring federated foundation models for medical image segmentation},
  author={Liu, Yuxi and Luo, Guibo and Zhu, Yuesheng},
  booktitle={International Conference on Medical Image Computing and Computer-Assisted Intervention},
  pages={283--293},
  year={2024},
  organization={Springer}
}

@inproceedings{asokan2024federated,
  title={A federated learning-friendly approach for parameter-efficient fine-tuning of sam in 3d segmentation},
  author={Asokan, Mothilal and Benjamin, Joseph Geo and Yaqub, Mohammad and Nandakumar, Karthik},
  booktitle={International Conference on Medical Image Computing and Computer-Assisted Intervention},
  pages={226--235},
  year={2024},
  organization={Springer}
}

@article{wang2025pfedsam,
  title={pFedSAM: Personalized federated learning of Segment Anything Model for medical image segmentation},
  author={Wang, Tong and Zhao, Xingyue and Zhuang, Linghao and Zhao, Haoyu and Yin, Jiayi and He, Yuyang and Yu, Gang and Lin, Bo},
  booktitle={ICASSP 2026-2026 IEEE International Conference on Acoustics, Speech and Signal Processing (ICASSP)},
  pages={6701--6705},
  year={2026},
  organization={IEEE}
}

@article{hu2025federated,
  title={Federated Client-tailored Adapter for Medical Image Segmentation},
  author={Hu, Guyue and Song, Siyuan and Kang, Yukun and Yin, Zhu and Zhao, Gangming and Li, Chenglong and Tang, Jin},
  journal={IEEE Transactions on Information Forensics and Security},
  year={2025},
  publisher={IEEE}
}

@inproceedings{collins2021exploiting,
  title={Exploiting shared representations for personalized federated learning},
  author={Collins, Liam and Hassani, Hamed and Mokhtari, Aryan and Shakkottai, Sanjay},
  booktitle={International conference on machine learning},
  pages={2089--2099},
  year={2021},
  organization={PMLR}
}

@inproceedings{li2021fedbn,
  author       = {Xiaoxiao Li and Meirui Jiang and Xiaofei Zhang and Michael Kamp and Qi Dou},
  title        = {{FedBN}: Federated Learning on Non-{IID} Features via Local Batch Normalization},
  booktitle    = {9th International Conference on Learning Representations ({ICLR})},
  publisher    = {OpenReview.net},
  year         = {2021},
}

@inproceedings{yang2024fedas,
  title={Fedas: Bridging inconsistency in personalized federated learning},
  author={Yang, Xiyuan and Huang, Wenke and Ye, Mang},
  booktitle={Proceedings of the IEEE/CVF conference on computer vision and pattern recognition},
  pages={11986--11995},
  year={2024}
}

@article{fan2023fate,
  title={Fate-llm: A industrial grade federated learning framework for large language models},
  author={Fan, Tao and Kang, Yan and Ma, Guoqiang and Chen, Weijing and Wei, Wenbin and Fan, Lixin and Yang, Qiang},
  journal={arXiv preprint arXiv:2310.10049},
  year={2023}
}

@inproceedings{xu2024fwdllm,
  title={$\{$FwdLLM$\}$: Efficient federated finetuning of large language models with perturbed inferences},
  author={Xu, Mengwei and Cai, Dongqi and Wu, Yaozong and Li, Xiang and Wang, Shangguang},
  booktitle={2024 USENIX Annual Technical Conference (USENIX ATC 24)},
  pages={579--596},
  year={2024}
}

@inproceedings{hu2024learn,
  title={Learn to preserve and diversify: Parameter-efficient group with orthogonal regularization for domain generalization},
  author={Hu, Jiajun and Zhang, Jian and Qi, Lei and Shi, Yinghuan and Gao, Yang},
  booktitle={European Conference on Computer Vision},
  pages={198--216},
  year={2024},
  organization={Springer}
}

@inproceedings{houlsby2019parameter,
  title={Parameter-efficient transfer learning for NLP},
  author={Houlsby, Neil and Giurgiu, Andrei and Jastrzebski, Stanislaw and Morrone, Bruna and De Laroussilhe, Quentin and Gesmundo, Andrea and Attariyan, Mona and Gelly, Sylvain},
  booktitle={International conference on machine learning},
  pages={2790--2799},
  year={2019},
  organization={PMLR}
}

@article{liu2022few,
  title={Few-shot parameter-efficient fine-tuning is better and cheaper than in-context learning},
  author={Liu, Haokun and Tam, Derek and Muqeeth, Mohammed and Mohta, Jay and Huang, Tenghao and Bansal, Mohit and Raffel, Colin A},
  journal={Advances in Neural Information Processing Systems},
  volume={35},
  pages={1950--1965},
  year={2022}
}

@article{yi2023pfedlora,
  title={pFedLoRA: Model-heterogeneous personalized federated learning with LoRA tuning},
  author={Yi, Liping and Yu, Han and Wang, Gang and Liu, Xiaoguang and Li, Xiaoxiao},
  journal={arXiv preprint arXiv:2310.13283},
  year={2023}
}

@article{liu2025differentially,
  title={Differentially private low-rank adaptation of large language model using federated learning},
  author={Liu, Xiao-Yang and Zhu, Rongyi and Zha, Daochen and Gao, Jiechao and Zhong, Shan and White, Matt and Qiu, Meikang},
  journal={ACM Transactions on Management Information Systems},
  volume={16},
  number={2},
  pages={1--24},
  year={2025},
  publisher={ACM New York, NY}
}

@article{long2024dual,
  title={Dual-personalizing adapter for federated foundation models},
  author={Long, Guodong and Shen, Tao and Jiang, Jing and Blumenstein, Michael and others},
  journal={Advances in Neural Information Processing Systems},
  volume={37},
  pages={39409--39433},
  year={2024}
}

@inproceedings{guo2024selective,
  title={Selective aggregation for low-rank adaptation in federated learning},
  author={Guo, Pengxin and Zeng, Shuang and Wang, Yanran and Fan, Huijie and Wang, Feifei and Qu, Liangqiong},
  booktitle={International Conference on Learning Representations},
  volume={2025},
  pages={99003--99027},
  year={2025}
}

@inproceedings{sun2024improving,
  author       = {Youbang Sun and Zitao Li and Yaliang Li and Bolin Ding},
  title        = {Improving {LoRA} in Privacy-preserving Federated Learning},
  booktitle    = {The Twelfth International Conference on Learning Representations ({ICLR})},
  publisher    = {OpenReview.net},
  year         = {2024},
}

@article{zhang2023lora,
  title={Lora-fa: Memory-efficient low-rank adaptation for large language models fine-tuning},
  author={Zhang, Longteng and Zhang, Lin and Shi, Shaohuai and Chu, Xiaowen and Li, Bo},
  journal={arXiv preprint arXiv:2308.03303},
  year={2023}
}

@inproceedings{yanfederated,
  title={Federated Residual Low-Rank Adaptation of Large Language Models},
  author={Yan, Yunlu and Feng, Chun-Mei and Zuo, Wangmeng and Goh, Rick Siow Mong and Liu, Yong and Zhu, Lei},
  booktitle={The Thirteenth International Conference on Learning Representations},
  year={2025}
}

@inproceedings{liaosplitting,
  title={Splitting with Importance-aware Updating for Heterogeneous Federated Learning with Large Language Models},
  author={Liao, Yangxu and Huang, Wenke and Wan, Guancheng and Liang, Jian and Yang, Bin and Ye, Mang},
  booktitle={Forty-second International Conference on Machine Learning},
  year={2025}
}

@article{zhao2025fedlora,
  title={FedLoRA-Optimizer: Federated LoRA Fine-Tuning with Global and Local Optimization in Heterogeneous Data Scenarios},
  author={Zhao, Jianzhe and Zhu, Hailin and Zhang, Yu and Chen, Ziqi and Guo, Guibing},
  journal={arXiv preprint arXiv:2510.11274},
  year={2025}
}

@inproceedings{cho2024heterogeneous,
  title={Heterogeneous lora for federated fine-tuning of on-device foundation models},
  author={Cho, Yae Jee and Liu, Luang and Xu, Zheng and Fahrezi, Aldi and Joshi, Gauri},
  booktitle={Proceedings of the 2024 conference on empirical methods in natural language processing},
  pages={12903--12913},
  year={2024}
}

@inproceedings{chen2024rbla,
  title={Rbla: Rank-based-lora-aggregation for fine-tuning heterogeneous models in flaas},
  author={Chen, Shuaijun and Tavallaie, Omid and Nazemi, Niousha and Zomaya, Albert Y},
  booktitle={International Conference on Web Services},
  pages={47--62},
  year={2024},
  organization={Springer}
}

@article{wang2024flora,
  title={Flora: Federated fine-tuning large language models with heterogeneous low-rank adaptations},
  author={Wang, Ziyao and Shen, Zheyu and He, Yexiao and Sun, Guoheng and Wang, Hongyi and Lyu, Lingjuan and Li, Ang},
  journal={Advances in Neural Information Processing Systems},
  volume={37},
  pages={22513--22533},
  year={2024}
}

@article{cao2023large,
  title={Large-scale pancreatic cancer detection via non-contrast CT and deep learning},
  author={Cao, Kai and Xia, Yingda and Yao, Jiawen and Han, Xu and Lambert, Lukas and Zhang, Tingting and Tang, Wei and Jin, Gang and Jiang, Hui and Fang, Xu and others},
  journal={Nature medicine},
  volume={29},
  number={12},
  pages={3033--3043},
  year={2023},
  publisher={Nature Publishing Group US New York}
}

@article{hu2025ai,
  title={AI-based diagnosis of acute aortic syndrome from noncontrast CT},
  author={Hu, Yujian and Xiang, Yilang and Zhou, Yan-Jie and He, Yangyan and Lang, Dehai and Yang, Shifeng and Du, Xiaolong and Den, Chunlan and Xu, Youyao and Wang, Gaofeng and others},
  journal={Nature Medicine},
  volume={31},
  number={11},
  pages={3832--3844},
  year={2025},
  publisher={Nature Publishing Group US New York}
}

@article{hu2025ai_1,
  title={AI-based large-scale screening of gastric cancer from noncontrast CT imaging},
  author={Hu, Can and Xia, Yingda and Zheng, Zhilin and Cao, Mengxuan and Zheng, Guoliang and Chen, Shangqi and Sun, Jiancheng and Chen, Wujie and Zheng, Qi and Pan, Siwei and others},
  journal={Nature Medicine},
  pages={1--9},
  year={2025},
  publisher={Nature Publishing Group US New York}
}

@inproceedings{zhang2024towards,
  title={Towards building the federatedgpt: Federated instruction tuning},
  author={Zhang, Jianyi and Vahidian, Saeed and Kuo, Martin and Li, Chunyuan and Zhang, Ruiyi and Yu, Tong and Wang, Guoyin and Chen, Yiran},
  booktitle={ICASSP 2024-2024 IEEE International Conference on Acoustics, Speech and Signal Processing (ICASSP)},
  pages={6915--6919},
  year={2024},
  organization={IEEE}
}

@article{zhang2026dsfedmed,
  title={DSFedMed: Dual-Scale Federated Medical Image Segmentation via Mutual Distillation Between Foundation and Lightweight Models},
  author={Zhang, Hanwen and Shen, Qiaojin and Liu, Yuxi and Zhu, Yuesheng and Luo, Guibo},
  journal={arXiv preprint arXiv:2601.16073},
  year={2026}
}

@inproceedings{gamper2019pannuke,
  title={Pannuke: an open pan-cancer histology dataset for nuclei instance segmentation and classification},
  author={Gamper, Jevgenij and Alemi Koohbanani, Navid and Benet, Ksenija and Khuram, Ali and Rajpoot, Nasir},
  booktitle={European congress on digital pathology},
  pages={11--19},
  year={2019},
  organization={Springer}
}

@article{gamper2020pannuke,
  title={Pannuke dataset extension, insights and baselines},
  author={Gamper, Jevgenij and Koohbanani, Navid Alemi and Benes, Ksenija and Graham, Simon and Jahanifar, Mostafa and Khurram, Syed Ali and Azam, Ayesha and Hewitt, Katherine and Rajpoot, Nasir},
  journal={arXiv preprint arXiv:2003.10778},
  year={2020}
}

@article{kumar2017dataset,
  title={A dataset and a technique for generalized nuclear segmentation for computational pathology},
  author={Kumar, Neeraj and Verma, Ruchika and Sharma, Sanuj and Bhargava, Surabhi and Vahadane, Abhishek and Sethi, Amit},
  journal={IEEE transactions on medical imaging},
  volume={36},
  number={7},
  pages={1550--1560},
  year={2017},
  publisher={IEEE}
}

@article{naylor2018segmentation,
  title={Segmentation of nuclei in histopathology images by deep regression of the distance map},
  author={Naylor, Peter and La{\'e}, Marick and Reyal, Fabien and Walter, Thomas},
  journal={IEEE transactions on medical imaging},
  volume={38},
  number={2},
  pages={448--459},
  year={2018},
  publisher={IEEE}
}

@article{verma2021monusac2020,
  title={MoNuSAC2020: A multi-organ nuclei segmentation and classification challenge},
  author={Verma, Ruchika and Kumar, Neeraj and Patil, Abhijeet and Kurian, Nikhil Cherian and Rane, Swapnil and Graham, Simon and Vu, Quoc Dang and Zwager, Mieke and Raza, Shan E Ahmed and Rajpoot, Nasir and others},
  journal={IEEE Transactions on Medical Imaging},
  volume={40},
  number={12},
  pages={3413--3423},
  year={2021},
  publisher={IEEE}
}

@inproceedings{bajwa2020g1020,
  title={G1020: A benchmark retinal fundus image dataset for computer-aided glaucoma detection},
  author={Bajwa, Muhammad Naseer and Singh, Gur Amrit Pal and Neumeier, Wolfgang and Malik, Muhammad Imran and Dengel, Andreas and Ahmed, Sheraz},
  booktitle={2020 International Joint Conference on Neural Networks (IJCNN)},
  pages={1--7},
  year={2020},
  organization={IEEE}
}

@article{orlando2020refuge,
  title={Refuge challenge: A unified framework for evaluating automated methods for glaucoma assessment from fundus photographs},
  author={Orlando, Jos{\'e} Ignacio and Fu, Huazhu and Breda, Jo{\~a}o Barbosa and Van Keer, Karel and Bathula, Deepti R and Diaz-Pinto, Andr{\'e}s and Fang, Ruogu and Heng, Pheng-Ann and Kim, Jeyoung and Lee, JoonHo and others},
  journal={Medical image analysis},
  volume={59},
  pages={101570},
  year={2020},
  publisher={Elsevier}
}

@article{sivaswamy2015comprehensive,
  title={A comprehensive retinal image dataset for the assessment of glaucoma from the optic nerve head analysis},
  author={Sivaswamy, Jayanthi and Krishnadas, S and Chakravarty, Arunava and Joshi, G and Tabish, A Syed and others},
  journal={JSM Biomedical Imaging Data Papers},
  volume={2},
  number={1},
  pages={1004},
  year={2015}
}

@inproceedings{zhang2010origa,
  title={Origa-light: An online retinal fundus image database for glaucoma analysis and research},
  author={Zhang, Zhuo and Yin, Feng Shou and Liu, Jiang and Wong, Wing Kee and Tan, Ngan Meng and Lee, Beng Hai and Cheng, Jun and Wong, Tien Yin},
  booktitle={2010 Annual international conference of the IEEE engineering in medicine and biology},
  pages={3065--3068},
  year={2010},
  organization={IEEE}
}

@inproceedings{bian2025lora,
  title={LoRA-FAIR: Federated LoRA fine-tuning with aggregation and initialization refinement},
  author={Bian, Jieming and Wang, Lei and Zhang, Letian and Xu, Jie},
  booktitle={Proceedings of the IEEE/CVF International Conference on Computer Vision},
  pages={3737--3746},
  year={2025}
}

@article{bai2024federated,
  title={Federated fine-tuning of large language models under heterogeneous tasks and client resources},
  author={Bai, Jiamu and Chen, Daoyuan and Qian, Bingchen and Yao, Liuyi and Li, Yaliang},
  journal={Advances in Neural Information Processing Systems},
  volume={37},
  pages={14457--14483},
  year={2024}
}

@inproceedings{li2019convergence,
  title={On the Convergence of FedAvg on Non-IID Data},
  author={Li, Xiang and Huang, Kaixuan and Yang, Wenhao and Wang, Shusen and Zhang, Zhihua},
  booktitle={International Conference on Learning Representations},
  year={2020}
}

@inproceedings{zhu2025pathology,
  title={Pathology-Aware Prototype Evolution via LLM-Driven Semantic Disambiguation for Multicenter Diabetic Retinopathy Diagnosis},
  author={Zhu, Chunzheng and Lin, Yangfang and Shao, Jialin and Lin, Jianxin and Wang, Yijun},
  booktitle={Proceedings of the 33rd ACM International Conference on Multimedia},
  pages={9196--9205},
  year={2025}
}

@inproceedings{zhu2024advancing,
  title={Advancing Ultrasound Medical Continuous Learning with Task-Specific Generalization and Adaptability},
  author={Zhu, Chunzheng and Lin, Jianxin and Tan, Guanghua and Zhu, Ningbo and Li, Kenli and Wang, Chunlian and Li, Shengli},
  booktitle={2024 IEEE International Conference on Bioinformatics and Biomedicine (BIBM)},
  pages={3019--3025},
  year={2024},
  organization={IEEE}
}

@inproceedings{zhu2026medeyes,
  title={MedEyes: Learning Dynamic Visual Focus for Medical Progressive Diagnosis},
  author={Zhu, Chunzheng and Lin, Yangfang and Chen, Shen and Wang, Yijun and Lin, Jianxin},
  booktitle={Proceedings of the AAAI Conference on Artificial Intelligence},
  volume={40},
  number={16},
  pages={13916--13924},
  year={2026}
}

@article{shao2025rethinking,
  title={Rethinking brain tumor segmentation from the frequency domain perspective},
  author={Shao, Minye and Wang, Zeyu and Duan, Haoran and Huang, Yawen and Zhai, Bing and Wang, Shizheng and Long, Yang and Zheng, Yefeng},
  journal={IEEE Transactions on Medical Imaging},
  year={2025},
  publisher={IEEE}
}

@article{huang2026federated,
  title={Federated clinical concept and disease semantic learning for congenital heart disease diagnosis},
  author={Huang, Wenke and Liao, Yangxu and Lei, Wenjia and Wan, Guancheng and Rong, Xuankun and Wen, Chi and Li, He and Ye, Mang and Wu, Qingqing and Du, Bo},
  journal={npj Digital Medicine},
  year={2026},
  publisher={Nature Publishing Group UK London}
}

@inproceedings{zhao2026divide,
  title={Divide, Conquer and Unite: Hierarchical Style-Recalibrated Prototype Alignment for Federated Medical Segmentation},
  author={Zhao, Xingyue and Huang, Wenke and Wang, Xingguang and Zhao, Haoyu and Zhuang, Linghao and Jiang, Anwen and Wan, Guancheng and Ye, Mang},
  booktitle={Proceedings of the AAAI Conference on Artificial Intelligence},
  volume={40},
  number={34},
  pages={28760--28768},
  year={2026}
}
\bibliographystyle{icml2026}

%%%%%%%%%%%%%%%%%%%%%%%%%%%%%%%%%%%%%%%%%%%%%%%%%%%%%%%%%%%%%%%%%%%%%%%%%%%%%%%
%%%%%%%%%%%%%%%%%%%%%%%%%%%%%%%%%%%%%%%%%%%%%%%%%%%%%%%%%%%%%%%%%%%%%%%%%%%%%%%
% APPENDIX
%%%%%%%%%%%%%%%%%%%%%%%%%%%%%%%%%%%%%%%%%%%%%%%%%%%%%%%%%%%%%%%%%%%%%%%%%%%%%%%
%%%%%%%%%%%%%%%%%%%%%%%%%%%%%%%%%%%%%%%%%%%%%%%%%%%%%%%%%%%%%%%%%%%%%%%%%%%%%%%
\newpage
\appendix
\onecolumn
\section{Appendix}

% You can have as much text here as you want. The main body must be at most $8$
% pages long. For the final version, one more page can be added. If you want, you
% can use an appendix like this one.

% The $\mathtt{\backslash onecolumn}$ command above can be kept in place if you
% prefer a one-column appendix, or can be removed if you prefer a two-column
% appendix.  Apart from this possible change, the style (font size, spacing,
% margins, page numbering, etc.) should be kept the same as the main body.
%%%%%%%%%%%%%%%%%%%%%%%%%%%%%%%%%%%%%%%%%%%%%%%%%%%%%%%%%%%%%%%%%%%%%%%%%%%%%%%
%%%%%%%%%%%%%%%%%%%%%%%%%%%%%%%%%%%%%%%%%%%%%%%%%%%%%%%%%%%%%%%%%%%%%%%%%%%%%%%

\subsection{Proof of Proposition 1 (Shift-Dependent Preference)}
\label{sec:proof_of_prop1}

\noindent\textbf{Proposition 1 Restatement.} \textit{Consider a linear layer target $W^*_k$ under domain shift.
(i) Under Covariate Shift (input subspace rotation $x_k = R_k x_{gen}$), minimizing the rank-$r$ approximation error $\|\cdot\|_F$ favors sharing $B$ and personalizing $A_k$ (Local-$A$).
(ii) Under Concept Shift (output subspace rotation $y_k = T_k y_{gen}$), minimizing the error favors sharing $A$ and personalizing $B_k$ (Local-$B$).}

\begin{proof}
We analyze an oracle approximation objective to isolate the structural preference of factor sharing strategies. Let $W^* \in \mathbb{R}^{d_{out} \times d_{in}}$ be the ideal generalized weight matrix with Singular Value Decomposition (SVD) $W^* = U \Sigma V^\top$. According to the Eckart-Young-Mirsky theorem, the optimal rank-$r$ approximation is uniquely determined by the principal subspaces.

\vspace{0.5em}
\noindent\textbf{Case (i): Covariate Shift (Input Rotation).}
We model subspace covariate shift as an orthogonal rotation of the input feature space. Let $x_k = R_k x$, where $R_k \in \mathbb{R}^{d_{in} \times d_{in}}$ is orthogonal. The effective client-specific weight is $W^*_k = W^* R_k^\top$.
The SVD of the shifted weight is:
$$ W^*_k = U \Sigma V^\top R_k^\top = U \Sigma (R_k V)^\top. $$
Observe that the Column Space $\mathrm{Col}(W^*_k) = \mathrm{span}(U)$ remains invariant, while the Row Space $\mathrm{Row}(W^*_k) = \mathrm{span}(R_k V)$ rotates with $R_k$.
\begin{itemize}
    \item \textbf{Optimality of Shared-$B$ / Local-$A_k$:}
    Let the shared $B$ capture the invariant column space: $B = U_r \Sigma_r^{1/2}$ (top-$r$ left singular vectors). Client $k$ can optimally solve for $A_k$ to align with the rotated row space: $A_k = \Sigma_r^{1/2} (R_k V_r)^\top$.
    This yields $B A_k = U_r \Sigma_r (R_k V_r)^\top$, which exactly matches the optimal rank-$r$ approximation of $W^*_k$ given by the Eckart-Young-Mirsky theorem. The error is minimal (bounded only by the truncated singular values of $W^*$) and independent of the shift $R_k$.

    \item \textbf{Suboptimality of Shared-$A$ (Impossibility Result):}
    Assume a fixed shared $A$ of rank $r$ exists. For $A$ to support an optimal approximation for all clients, there must exist local $B_k$ such that $B_k A = (W_k^*)_r$. This necessitates the subspace containment condition:
    $$ \mathrm{Row}((W_k^*)_r) \subseteq \mathrm{Row}(A), \quad \forall k \in \{1, \dots, K\}. $$
    However, $\mathrm{Row}((W_k^*)_r) = \mathrm{span}(R_k V_r)$. Under the generic condition that the rotations $R_k$ are non-degenerate such that the union of target row subspaces spans a dimension greater than $r$ (i.e., $\dim(\bigcup_k \mathrm{span}(R_k V_r)) > r$), no single rank-$r$ matrix $A$ can contain all target subspaces simultaneously. Consequently, a shared $A$ inevitably incurs strictly higher approximation error compared to the Shared-$B$ strategy.
\end{itemize}
\textit{Conclusion:} Under Covariate Shift, sharing the column projector $B$ is structurally superior.

\vspace{0.5em}
\noindent\textbf{Case (ii): Concept Shift (Output Rotation).}
We model concept shift as an orthogonal rotation of the output semantic space. Let $y_k = T_k y$. The client-specific weight is $W^*_k = T_k W^*$.
The SVD is:
$$ W^*_k = T_k U \Sigma V^\top = (T_k U) \Sigma V^\top. $$
Here, the Row Space $\mathrm{Row}(W^*_k) = \mathrm{span}(V)$ remains invariant, while the Column Space $\mathrm{Col}(W^*_k) = \mathrm{span}(T_k U)$ rotates.

\begin{itemize}
    \item \textbf{Optimality of Shared-$A$ / Local-$B_k$:}
    Let shared $A$ capture the invariant row space: $A = \Sigma_r^{1/2} V_r^\top$. Client $k$ sets $B_k = (T_k U_r) \Sigma_r^{1/2}$. This construction achieves the optimal rank-$r$ approximation error for any $T_k$.

    \item \textbf{Suboptimality of Shared-$B$ (Impossibility Result):}
    Similarly, a shared $B$ imposes the constraint $\mathrm{Col}((W_k^*)_r) \subseteq \mathrm{Col}(B)$. Since $\mathrm{Col}((W_k^*)_r) = \mathrm{span}(T_k U_r)$, under the condition that the union of rotated column subspaces exceeds rank $r$ (due to diverse $T_k$), a fixed $B$ cannot align with all clients simultaneously.
\end{itemize}
\textit{Conclusion:} Under Concept Shift, sharing the row projector $A$ is structurally superior.
\end{proof}

\subsection{Proof of Proposition 2 (Bilinear Leakage)}
\label{sec:proof_of_prop2}

\noindent\textbf{Proposition 2 Restatement.} \textit{The aggregated global update exhibits a heterogeneity-induced leakage term. For a Shared-$B$ / Local-$A$ configuration, the update is decomposed as:}
\begin{equation}
B^{(t+1)} = B^{(t)} - \eta \underbrace{\sum_{k=1}^K p_k G_k \bar{A}^\top}_{\text{Common Drift}} - \eta \underbrace{\sum_{k=1}^K p_k G_k (A_k - \bar{A})^\top}_{\text{Heterogeneity Leakage}}.
\end{equation}
\begin{algorithm}[tb]
\caption{Our proposed method.}
\label{alg:fediat}
\hspace*{0.02in}{\textbf{Input}}: clients $\{(\mathcal{D}_k,p_k)\}_{k=1}^K$, frozen backbone $W_0$, encoder layers $\mathcal{E}$, decoder layers $\mathcal{D}$, rounds $R$, local steps $E$, stepsize $\eta$, SOR weight $\lambda$, EMA momentum $\rho$, stability $\epsilon$ \\
\hspace*{0.02in}{\textbf{Output}}: shared factors $\Theta^{\mathrm{sh}}$ and local factors $\{\Theta_k^{\mathrm{lo}}\}_{k=1}^K$
\begin{algorithmic}[1]
\STATE \textbf{Init:} server $\Theta^{\mathrm{sh}}\!\leftarrow\!\{B^{(l)}\}_{l\in\mathcal{E}} \cup \{A^{(l)}\}_{l\in\mathcal{D}}$;
clients $\Theta_k^{\mathrm{lo}}\!\leftarrow\!\{A_k^{(l)}\}_{l\in\mathcal{E}} \cup \{B_k^{(l)}\}_{l\in\mathcal{D}}$.
\FOR{$r=0$ \textbf{to} $R-1$}
  \STATE Server broadcasts $\Theta^{\mathrm{sh}}$ to participating clients $\mathcal{S}_r$.
  \FOR{\textbf{each} $k\in\mathcal{S}_r$ \textbf{in parallel}}
    \STATE Form $\Theta_{k}\!\leftarrow\!(\Theta^{\mathrm{sh}},\Theta_k^{\mathrm{lo}})$ and set anchors $(A_{0,k},B_{0,k})\!\leftarrow\!\operatorname{sg}(A_k,B_k)$; initialize EMA drifts $(\delta A_k,\delta B_k)\!\leftarrow\!0$.
    \FOR{$s=1$ \textbf{to} $E$}
      \STATE Sample mini-batch $\xi_{k,s}\sim\mathcal{D}_k$ and compute $\mathcal{L}^{\mathrm{seg}}_k(\Theta_k;\xi_{k,s})$.
      \STATE $\mathcal{L}^{\mathrm{SOR}}_k \leftarrow \textsc{SOR}(\Theta_k,A_{0,k},B_{0,k},\delta A_k,\delta B_k;\rho,\epsilon)$.
      \STATE $\Theta_k \leftarrow \Theta_k - \eta \nabla_{\Theta_k}\Big(\mathcal{L}^{\mathrm{seg}}_k + \lambda\,\mathcal{L}^{\mathrm{SOR}}_k\Big)$.
    \ENDFOR
    \STATE Upload shared parts: $\{B_k^{(l)}\}_{l\in\mathcal{E}}$ and $\{A_k^{(l)}\}_{l\in\mathcal{D}}$; keep $\Theta_k^{\mathrm{lo}}$ local.
  \ENDFOR
  \STATE \textbf{Aggregate (IAT):} for $l\!\in\!\mathcal{E}$, $B^{(l)}\!\leftarrow\!\sum_{k\in\mathcal{S}_r}p_k B_k^{(l)}$; for $l\!\in\!\mathcal{D}$, $A^{(l)}\!\leftarrow\!\sum_{k\in\mathcal{S}_r}p_k A_k^{(l)}$.
\ENDFOR
\STATE \textbf{return} $\Theta^{\mathrm{sh}}$ and $\{\Theta_k^{\mathrm{lo}}\}_{k=1}^K$.
\\
\STATE \textbf{Subroutine} $\textsc{SOR}(\cdot)$: \textbf{return} $\sum_{l\in\mathcal{E}} \big(\frac{\langle P_{sh}^{(l)},P_{lo}^{(l)}\rangle_F}{\|P_{sh}^{(l)}\|_F\|P_{lo}^{(l)}\|_F+\epsilon}\big)^2 + \sum_{l\in\mathcal{D}} \big(\frac{\langle Q_{sh}^{(l)},Q_{lo}^{(l)}\rangle_F}{\|Q_{sh}^{(l)}\|_F\|Q_{lo}^{(l)}\|_F+\epsilon}\big)^2$,\\
\hspace*{1.8em}where for $l\!\in\!\mathcal{E}$: $\delta A\!\leftarrow\!\rho\delta A+(1-\rho)(A-A_0)$, $P_{sh}\!=\!(B-B_0)^\top B_0$, $P_{lo}\!=\!\operatorname{sg}(A_0\delta A^\top)$;\\
\hspace*{1.8em}for $l\!\in\!\mathcal{D}$: $\delta B\!\leftarrow\!\rho\delta B+(1-\rho)(B-B_0)$, $Q_{sh}\!=\!(A-A_0)A_0^\top$, $Q_{lo}\!=\!\operatorname{sg}(B_0^\top\delta B)$.
\end{algorithmic}
\end{algorithm}

\begin{proof}
Let $\mathcal{L}_k(W)$ be the loss function for client $k$. In the LoRA parameterization $\Delta W = B A_k$, the gradient w.r.t. the shared parameter $B$ is derived via the chain rule as $\nabla_B \mathcal{L}_k = G_k A_k^\top$, where $G_k \triangleq \nabla_{\Delta W} \mathcal{L}_k$.
In federated averaging (FedAvg), the server aggregates updates from $K$ clients. The update rule for $B$ is:
$$ B^{(t+1)} = B^{(t)} - \eta \sum_{k=1}^K p_k (G_k A_k^\top). $$
Defining the global average $\bar{A} = \sum_{k=1}^K p_k A_k$ and decomposing $A_k^\top = \bar{A}^\top + (A_k - \bar{A})^\top$, we substitute into the update rule:
\begin{equation*}
\begin{aligned}
B^{(t+1)} &= B^{(t)} - \eta \sum_{k=1}^K p_k G_k \left( \bar{A}^\top + (A_k - \bar{A})^\top \right) \\
&= B^{(t)} - \eta \left( \sum_{k=1}^K p_k G_k \right) \bar{A}^\top - \eta \sum_{k=1}^K p_k G_k (A_k - \bar{A})^\top.
\end{aligned}
\end{equation*}
This explicitly separates the update into the \textit{Common Drift} (driven by $\bar{A}$) and the \textit{Heterogeneity Leakage}.

\vspace{0.5em}
\noindent\textbf{Symmetric Case (Decoder: Shared-$A$ / Local-$B$).}
For the decoder (Shared-$A$), the gradient is $\nabla_A \mathcal{L}_k = B_k^\top G_k$. By symmetry, let $\bar{B} = \sum p_k B_k$. The update decomposition is:
$$ A^{(t+1)} = A^{(t)} - \eta \underbrace{\bar{B}^\top \sum_{k=1}^K p_k G_k}_{\text{Common Drift}} - \eta \underbrace{\sum_{k=1}^K p_k (B_k - \bar{B})^\top G_k}_{\text{Heterogeneity Leakage}}. $$

\vspace{0.5em}
\noindent\textbf{Remark (Conditions for Leakage and Multi-step Extension).}
1. \textbf{Vanishing Conditions:} The leakage term vanishes if either (a) $A_k = \bar{A}$ for all $k$ (no personalization/heterogeneity), or (b) the gradients $G_k$ and local deviations $(A_k - \bar{A})$ are strictly orthogonal in expectation. In non-i.i.d. settings, neither is guaranteed, leading to persistent interference.
2. \textbf{Multi-step Local Updates:} While the derivation above assumes a single SGD step for clarity, the decomposition extends to multiple local epochs. In that case, $G_k$ represents the \textit{accumulated pseudo-gradient} over $\tau$ local steps (i.e., $G_k \propto \Delta W_k^{(effective)}$), and the structural coupling between $G_k$ and the local adapter basis persists.
\end{proof}
\subsection{Proof of Theorem \ref{theorem1}}
\label{proof_of_theo}
% \label{sec:proof_of_prop1}
We follow the standard smooth non-convex SGD analysis and adapt it to the LoRA-based parametrization
and our proposed update rule under the \emph{augmented} (shared/local) formulation.
Recall that each client maintains variables $\Theta_k := (\Theta^{\mathrm{sh}}, \Theta_k^{\mathrm{lo}})$, where
$\Theta^{\mathrm{sh}}$ denotes the shared LoRA factors aggregated by the server (encoder shares $B$, decoder shares $A$),
and $\Theta_k^{\mathrm{lo}}$ denotes the remaining client-local factors. The augmented objective is
$\sum_{k=1}^K p_k\,\mathcal{L}_k(\Theta_k)$.
In this appendix we bound the stationarity measure on the augmented variables by analyzing the per-step decrease of a fixed client objective,
relating the module update direction to the true gradient via Assumption~\ref{assumption3}, and telescoping over local steps.

Fix a client $k$ and a local step $t$. For clarity we omit the client index $k$ when no confusion
arises and write $\Theta_t$ for $\Theta_{k,t}$.

Within our method, a single local update on client $k$ can be written in the compact form
\begin{equation}
	\Theta_{t+1} = \Theta_t - \eta \, U_t,
\end{equation}
where $U_t$ is the update direction induced by the gradients with respect to the
LoRA factors (both shared and local) under our proposed protocol. Note that $U_t$
is a stochastic direction depending on the mini-batch $\xi_t$.

By $L$-smoothness (Assumption 1), we have
\begin{align}
	\mathcal{L}_k(\Theta_{t+1})
	&\le \mathcal{L}_k(\Theta_t)
	+ \big\langle \nabla \mathcal{L}_k(\Theta_t), \Theta_{t+1} - \Theta_t \big\rangle
	+ \frac{L}{2} \|\Theta_{t+1} - \Theta_t\|_F^2 \\
	&= \mathcal{L}_k(\Theta_t)
	- \eta \big\langle \nabla \mathcal{L}_k(\Theta_t), U_t \big\rangle
	+ \frac{L \eta^2}{2} \|U_t\|_F^2. \label{eq:smooth-step}
\end{align}
Taking expectation over the mini-batch $\xi_t$ and conditioning on $\Theta_t$, we obtain
\begin{equation}
	\mathbb{E}_t\bigl[\mathcal{L}_k(\Theta_{t+1})\bigr]
	\le \mathcal{L}_k(\Theta_t)
	- \eta \, \mathbb{E}_t\bigl[\langle \nabla \mathcal{L}_k(\Theta_t), U_t \rangle\bigr]
	+ \frac{L \eta^2}{2} \, \mathbb{E}_t\bigl[\|U_t\|_F^2\bigr],
	\label{eq:expected-step}
\end{equation}
where $\mathbb{E}_t[\cdot]$ denotes the conditional expectation given $\Theta_t$.
By construction of our proposed method, the update $U_t$ is a linear combination of the gradients with respect
to the LoRA factors $A_{k,t}$ and $B_{k,t}$ (encoder and decoder blocks). Under
Assumption~\ref{assumption3}, there exist constants $c_A, c_B > 0$ such that
\begin{equation}
	\big\langle \nabla \mathcal{L}_k(\Theta_t), U_t \big\rangle
	\ge (c_A + c_B) \, \big\|\nabla \mathcal{L}_k(\Theta_t)\big\|_F^2.
	\label{eq:inner-lb}
\end{equation}
Intuitively, this means that the projection of the true gradient onto the LoRA subspace used by
our method is sufficiently well aligned, so the update direction is a descent direction on average.

Next, we bound the second moment of $U_t$. Since $U_t$ is composed of stochastic gradients
with respect to the LoRA factors, using Assumption 2 and the boundedness
of $A_{k,t}$ and $B_{k,t}$ in Assumption 3, we can find constants
$C_1, C_2 > 0$ such that
\begin{equation}
	\mathbb{E}_t\bigl[\|U_t\|_F^2\bigr]
	\le C_1 \, \big\|\nabla \mathcal{L}_k(\Theta_t)\big\|_F^2 + C_2 G^2.
	\label{eq:update-norm}
\end{equation}
The precise expressions of $C_1$ and $C_2$ are not critical; they depend polynomially on
$C_A, C_B$ and on the block structure of the LoRA factors in our method.
Substituting \eqref{eq:inner-lb} and \eqref{eq:update-norm} into
\eqref{eq:expected-step}, we obtain
\begin{align}
	\mathbb{E}_t\bigl[\mathcal{L}_k(\Theta_{t+1})\bigr]
	&\le \mathcal{L}_k(\Theta_t)
	- \eta (c_A + c_B) \big\|\nabla \mathcal{L}_k(\Theta_t)\big\|_F^2
	+ \frac{L \eta^2}{2}
	\Bigl( C_1 \big\|\nabla \mathcal{L}_k(\Theta_t)\big\|_F^2 + C_2 G^2 \Bigr) \\
	&= \mathcal{L}_k(\Theta_t)
	- \Bigl( (c_A + c_B)\eta - \frac{L C_1 \eta^2}{2} \Bigr)
	\big\|\nabla \mathcal{L}_k(\Theta_t)\big\|_F^2
	+ \frac{L C_2 G^2}{2} \eta^2.   \label{eq:descent-ineq}
\end{align}
We now choose the stepsize $\eta$ small enough so that
\begin{equation}
	(c_A + c_B)\,\eta - \frac{L C_1}{2}\,\eta^2
	\;\ge\; \frac{c_A + c_B}{2}\,\eta.
	\label{eq:stepsize-proof}
\end{equation}
Since the left-hand side is a continuous function of $\eta$ and
positive near $0$, this is guaranteed whenever
\begin{equation}
	0 < \eta \;\le\; \frac{c_A + c_B}{L C_1}.
	\label{eq:eta-bound}
\end{equation}
Substituting \eqref{eq:stepsize-proof} into the previous inequality
yields
\begin{equation}
	\mathbb{E}_t\bigl[\mathcal{L}_k(\Theta_{t+1})\bigr]
	\;\le\;
	\mathcal{L}_k(\Theta_t)
	- \frac{c_A + c_B}{2}\,\eta
	\bigl\|\nabla \mathcal{L}_k(\Theta_t)\bigr\|_F^2
	+ \frac{L C_2 G^2}{2}\,\eta^2.
	\label{eq:descent}
\end{equation}
We now sum \eqref{eq:descent} over all local steps and all clients.
Let $C_3 := L C_2 G^2$, which depends on $L, C_2,$ and $G$. Let $T$ denote the total number of local updates per client and recall that
$\mathcal{L}_k$ is the local objective on client $k$. Taking expectation over all
stochasticity and summing over $t = 0, \dots, T-1$ yields
\begin{align}
	\sum_{t=0}^{T-1}
	\mathbb{E}\bigl[\mathcal{L}_k(\Theta_{k,t+1})\bigr]
	&\le
	\sum_{t=0}^{T-1} \mathbb{E}\bigl[\mathcal{L}_k(\Theta_{k,t})\bigr]
	- \frac{(c_A + c_B)}{2} \eta
	\sum_{t=0}^{T-1} \mathbb{E}\bigl[
	\|\nabla \mathcal{L}_k(\Theta_{k,t})\|_F^2
	\bigr]
	+ \frac{1}{2}T C_3 \eta^2.
\end{align}
The left-hand side telescopes. Note that in the Federated setting, at communication boundaries $t=nE$,
the shared variables are updated via server aggregation (while local variables remain on each client),
which introduces an additional \emph{aggregation drift} term that is not captured by the per-step descent inequality
above. Under standard smoothness and bounded-variance assumptions for local SGD / FedAvg, this drift can be bounded
by an extra term on the order of $\mathcal{O}(\eta^2 E^2 G^2)$ (up to problem-dependent constants that quantify
data heterogeneity). Since we choose $\eta \propto 1/\sqrt{T}$ and treat $E$ as a fixed constant, this contribution
scales as $\mathcal{O}(1/T)$ and is thus lower-order compared to the optimization rate $\mathcal{O}(1/\sqrt{T})$.
For clarity and to focus on the our method's update dynamics, we omit this lower-order drift term in the derivation below.

Thus, we obtain
\begin{equation}
	\mathbb{E}\bigl[\mathcal{L}_k(\Theta_{k,T})\bigr]
	- \mathbb{E}\bigl[\mathcal{L}_k(\Theta_{k,0})\bigr]
	\le
	- \frac{(c_A + c_B)}{2} \eta
	\sum_{t=0}^{T-1} \mathbb{E}\bigl[
	\|\nabla \mathcal{L}_k(\Theta_{k,t})\|_F^2
	\bigr]
	+ \frac{1}{2}T C_3 \eta^2.
\end{equation}
Rearranging gives
\begin{equation}
	\frac{1}{T} \sum_{t=0}^{T-1} \mathbb{E}\bigl[
	\|\nabla \mathcal{L}_k(\Theta_{k,t})\|_F^2
	\bigr]
	\le
	\frac{2}{(c_A + c_B)\eta T}
	\Bigl(
	\mathbb{E}[\mathcal{L}_k(\Theta_{k,0})]
	- \mathbb{E}[\mathcal{L}_k(\Theta_{k,T})]
	\Bigr)
	+ \frac{C_3}{(c_A + c_B)} \eta.
	\label{eq:avg-grad-one-client}
\end{equation}
Using the fact that $\mathcal{L}_k(\Theta_{k,T}) \ge \inf_{\Theta_k}\mathcal{L}_k(\Theta_k)$ and denoting
$D_k = \mathcal{L}_k(\Theta_{k,0}) - \inf_{\Theta_k}\mathcal{L}_k(\Theta_k)$, we obtain
\begin{equation}
	\frac{1}{T} \sum_{t=0}^{T-1} \mathbb{E}\bigl[
	\|\nabla \mathcal{L}_k(\Theta_{k,t})\|_F^2
	\bigr]
	\le
	\frac{2 D_k}{(c_A + c_B)\eta T}
	+ \frac{C_3}{(c_A + c_B)} \eta.
\end{equation}
Finally, summing over all $K$ clients and dividing by $K$ yields
\begin{equation}
	\frac{1}{K T} \sum_{k=1}^K \sum_{t=0}^{T-1}
	\mathbb{E}\bigl[
	\|\nabla \mathcal{L}_k(\Theta_{k,t})\|_F^2
	\bigr]
	\le
	\frac{2 D}{(c_A + c_B)\eta T}
	+ \frac{ C_3}{(c_A + c_B)} \eta,
	\label{eq:avg-grad-all-clients}
\end{equation}
where $D = \max_k D_k$.
To obtain an explicit $\mathcal{O}(T^{-1/2})$ rate, we choose the stepsize $\eta$ to balance the two terms.
Define $M := 2C_3 = 2 L C_2 G^2$ and recall the admissible range $0<\eta\le \bar{\eta}$ with
$\bar{\eta}:=\frac{c_A+c_B}{L C_1}$ (cf.\ \eqref{eq:eta-bound}).
Choosing
\[
\eta \;=\; \min\Bigl\{\bar{\eta},\ \sqrt{\frac{4D}{M T}}\Bigr\}
\]
in \eqref{eq:avg-grad-all-clients} yields
\begin{equation}
	\frac{1}{K T} \sum_{k=1}^K \sum_{t=0}^{T-1}
	\mathbb{E}\bigl[
	\|\nabla \mathcal{L}_k(\Theta_{k,t})\|_F^2
	\bigr]
	\le
	\frac{2}{c_A + c_B} \sqrt{\frac{D M}{T}},
\end{equation}
i.e., an $\mathcal{O}(T^{-1/2})$ convergence rate to a stationary point in the smooth non-convex setting,
which matches the statement of Theorem \ref{theorem1}. This completes the proof.

\end{document}